\documentclass{article}

     \PassOptionsToPackage{numbers, compress}{natbib}
\usepackage[preprint]{neurips_2026}

\usepackage[utf8]{inputenc} 
\usepackage[T1]{fontenc}    
\usepackage{hyperref}       
\usepackage{url}            
\usepackage{booktabs}       
\usepackage{amsfonts}       
\usepackage{amsmath}
\usepackage{amssymb}
\usepackage{amsthm}
\usepackage{nicefrac}       
\usepackage{microtype}      
\usepackage{xcolor}
\usepackage{mathtools}
\usepackage{comment}
\usepackage{graphicx} 
\usepackage{subfigure} 
\usepackage{pdflscape}
\usepackage{algorithm}
\usepackage{algorithmic}
\usepackage{bm}
\usepackage{soul}
\usepackage{listings}
\usepackage{bbm}
\usepackage{multirow}
\usepackage{multicol}

\usepackage{float}
\usepackage{wrapfig}
\usepackage{framed}

\usepackage{rotating}

\usepackage{wrapfig}

\newcommand\Tstrut{\rule{0pt}{2.2ex}}         

\newtheorem{theorem}{Theorem}[section]

\newcommand{\INDSTATE}[1][1]{\STATE\hspace{#1\algorithmicindent}}

\theoremstyle{plain}

\newtheorem{lemma}[theorem]{Lemma}

\theoremstyle{definition}

\theoremstyle{remark}

\title{
Large-scale Score-based Variational Posterior Inference for Bayesian Deep Neural Networks
}

\author{%
  Minyoung Kim\thanks{\url{https://minyoungkim21.github.io/}} \\
  Samsung AI Center Cambridge\\
  United Kingdom \\
  \texttt{mikim21@gmail.com} \\
}

\begin{document}

\vspace{-1.0em}

\maketitle

\vspace{-1.0em}

\begin{abstract}
\vspace{-1.0em}
Bayesian (deep) neural networks (BNN) are often more attractive than the vanilla point-estimate deep learning in various aspects including uncertainty quantification, robustness to noise, resistance to overfitting, and more. The variational inference (VI) is one of the most widely adopted approximate inference methods. Whereas the ELBO-based variational free energy method is a dominant choice in the literature, in this paper we introduce a score-based alternative for BNN variational inference. Score-based VI can address the known issue of mode collapsing in ELBO-based VI. Although several score-based VI methods have been proposed in the community, most are not adequate for large-scale BNNs for various computational and technical reasons. We propose a novel scalable VI method where the learning objective combines the score matching loss and the proximal penalty term in iterations, which helps our method avoid the reparametrized sampling, and allows for noisy unbiased mini-batch scores through stochastic gradients. This in turn makes our method scalable to large-scale neural networks including Vision Transformers. 
On several benchmarks including visual recognition and time-series forecasting with large-scale deep networks, we empirically show the effectiveness of our approach.
\end{abstract}

\section{Introduction}\label{sec:intro}
\vspace{-1.0em}
In machine learning, Bayesian modeling is a very powerful and promising tool in expressing domain knowledge in the form of prior, where based on the Bayes rule the posterior inference is a principled machinery of reasoning on the most plausible possible states of the unknowns~\citep{mackay,gha}. 
Bayesian neural networks (BNN)~\citep{buntine,mackay92,neal95} have gained significant attention over the last decade along with the trend of the popularity of deep learning. BNNs are Bayesian models where the parameters of the neural networks are treated as random variables, prior preferences on the parameters are imposed in the form of prior distributions, and the predictions of the neural networks are linked with problem-specific losses as likelihood functions. 
Like most non-trivial Bayesian models, however, posterior inference in BNNs is generally infeasible, 
and several approximate posterior inference methods have been proposed including: variational inference~\citep{bayes_by_backprop}, MC-dropout~\citep{mc_dropout}, Laplace approximation~\citep{pmlr-v139-immer21a}, sampling-based methods~\citep{sgld,pmlr-v119-wenzel20a}, 
the ensemble methods~\citep{ensemble}, to name just a few.

In this paper we are particularly interested in variational inference for large-scale deep networks and data. Variational inference (VI) is one of the most widely adopted methods that aims to approximate the target posterior by a distribution from a tractable density family, 
turning probabilistic inference into an optimization problem~\citep{jordan99,wainwright08,blei17}. 
Perhaps the two most popular VI approaches are the variational free energy methods (also known as ELBO) that minimize the KL divergence~\citep{bayes_by_backprop,advi} and the score-based approaches~\citep{fisher,score} that aims to directly minimize the difference between the target score and model's score. 
Although the ELBO-based 
method is a dominant choice, and has been applied successfully to large-scale BNNs~\citep{bayesdll}, it is known to have the mode collapsing issue due to its reverse KL objective. Score-based VI can address the issue, however, the previous score-matching approaches may not be directly applicable to large-scale BNNs for various difficulties including: computationally prohibitive Hessian through neural nets with inversion and being unable to support mini-batch stochastic gradients. 
In this paper we aim to propose a novel score-based alternative for large-scale BNN variational inference. 

Our proposed method is motivated by the recent Gaussian score matching (GSM) methods~\citep{gsmvi,bam}. Under the assumption of Gaussian variational density family, they admit iterative closed forms with several merits, for instance, faster convergence in practice than ELBO-based methods with strong theoretical guarantee. 
However, there are several drawbacks. Firstly, they are limited to small-scale VI problems due to matrix inverse or multiplication operations that take cubic or quadratic cost in state dimension~\citep{bam} while requiring access to {\em exact} scores and being unable to deal with noisy mini-batch stochastic-gradient scores. The latter implies that in BNNs, we need to access entire data/evidence for a single update (e.g., waiting for a full epoch of mini-batch iterations), making GSM inapplicable to large-scale data with large networks like Transformers~\citep{vaswani2017attention,vit} or ResNets~\citep{resnet}. 
Secondly, the variational density family in GSM is confined to Gaussians, which may lead to suboptimal performance if the true posterior is substantially different from Gaussians. 
%

We propose a novel score-matching VI algorithm 
applicable to large-scale BNN posterior inference. 
Our optimization objective combines the score-matching loss and the proximal penalty term in iterations, which helps our method avoid the reparametrized sampling, and allows for noisy unbiased mini-batch scores through stochastic gradients. 
Unlike GSM, our approach 
is scalable to large-scale problems including Bayesian deep models with Vision Transformers (ViT)~\citep{vit}. 
%
%
On several benchmarks ranging from visual recognition with 
large-scale ResNets~\citep{resnet} and ViT~\citep{vit} to time-series forecasting with more domain-tailored deep networks~\citep{koopa}, 
we show effectiveness of our approach.

\section{Problem Setup and Background}\label{sec:setup}
\vspace{-1.0em}
Generally we consider a target distribution 
denoted by $\pi(\theta)$ where $\theta\in\mathbb{R}^d$ is the state random variable. 
We assume that $\pi(\theta)$ is intractable in that sampling from it is infeasible while computing the log-density $\log \pi(\theta)$ is only feasible up to some unknown constant. 
{\setlength{\abovedisplayskip}{3pt}
\setlength{\belowdisplayskip}{3pt}
\begin{align}
\pi(\theta) = \frac{\Phi(\theta)}{Z}, \ \ \ \ Z = \int \Phi(\theta) d\theta
\label{eq:energy_form}
\end{align}
}
where $\Phi(\theta)\!\geq\!0$ is the potential function 
that is easy to compute, but the normalizing constant 
$Z$ is 
infeasible. A good example which we focus on in this paper is the Bayesian posterior inference: 
{\setlength{\abovedisplayskip}{3pt}
\setlength{\belowdisplayskip}{3pt}
\begin{align}
p(\theta|D) = \frac{p(\theta) p(D|\theta)}{p(D)}, \ \ p(D) = \int\! p(\theta) p(D|\theta) d\theta
\label{eq:bayesian_posterior}
\end{align}
}
where $D$ is observed data (i.e., evidence). 
By construction, the prior $p(\theta)$ and the likelihood $p(D|\theta)$ can be made easy to evaluate.  The potential $\Phi(\theta):= p(\theta) p(D|\theta)$ is then tractable, but the normalizer $Z := p(D)$ is often infeasible to compute. 
In this setup, however, the {\em score} of the target $s(\theta) := \nabla_\theta \log \pi(\theta) = \nabla_\theta \log \Phi(\theta)$ can be computed tractably where $\pi(\theta) := p(\theta|D)$. 
In BNNs, $\theta$ represents the parameters of the underlying neural network.


\textbf{Variational inference (VI)} aims to approximate the target $\pi(\theta)$ by a density function from some tractable family $\mathcal{Q}$ (e.g., Gaussian). That is, $q_\lambda(\theta) \approx \pi(\theta)$ where $q_\lambda(\theta)\in\mathcal{Q}$ is a parametric density function with parameters $\lambda$ (e.g., for the Gaussian, $q_\lambda(\theta) = \mathcal{N}(\theta; m, V)$ with $\lambda=(m,V)$). This turns the approximate inference into an optimization problem of finding the best $\lambda$. 
We review the two most popular VI algorithms, ELBO-based (Sec.~\ref{sec:elbo_bg}) and score-based (Sec.~\ref{sec:smvi_bg}) ones. 

\textbf{Noisy or mini-batch scores.}
In practical large-scale data scenarios, it is often the case that we only have a (noisy) mini-batch version of the score, denoted as $\hat{s}(\theta)$, instead of the true one $s(\theta) := \nabla_\theta \log \pi(\theta)$. Usually we have the unbiasedness property, i.e., $\mathbb{E}[\hat{s}(\theta)] = s(\theta)$ where the expectation is taken over the sampling randomness. In Bayesian inference, it originates from the mini-batch likelihood evaluation with large data/evidence $D$~\citep{doubly_stoch}. 
When $D$ is large, the likelihood $\log p(D|\theta)$ may not be computed exactly since it has to wait for the entire sweep through the data; Instead one can take a mini-batch $B$, a small random subset of $D$ with $|B|\!\ll\!|D|$ and approximate $\log p(D|\theta) \approx \frac{|D|}{|B|} \log p(B|\theta)$ as an unbiased estimate. 

Although the noisy scores can be handled naturally in ELBO-based VI with its stochastic gradient version, we find that existing score-matching VI methods are unable to deal with noisy scores properly as discussed in Sec.~\ref{sec:smvi_bg}. In our proposed stochastic-gradient score-matching algorithm (Sec.~\ref{sec:main}) we handle the noisy mini-batch scores effectively, which is one of our main contributions in this paper.


\subsection{Background on ELBO-based VI
}\label{sec:elbo_bg}

The variational density $q_\lambda(\theta)$ is selected by minimizing the (reverse) KL divergence $\textrm{KL}(q_\lambda(\theta)||\pi(\theta))$. This is equivalent to minimizing the so-called {\em Negative ELBO}:
\begin{align}
\textrm{Negative-ELBO} := \mathbb{E}_{q_\lambda(\theta)}[\log q_\lambda(\theta) - \log \Phi(\theta)]
\label{eq:negative_elbo}
\end{align}
The automatic differentiation VI (ADVI)~\citep{advi} solves minimization of (\ref{eq:negative_elbo}) using Monte-Carlo estimation with the reparametrized sampling trick~\citep{vae14,doubly_stoch,rezende}. 
More specifically, denoting by $\{\theta^{(i)}(\lambda)\}_{i=1}^S$ the $S$ i.i.d.~reparametrized samples from $q_\lambda(\theta)$ (note: the samples are functions of $\lambda$), the negative ELBO objective is approximated as: $\frac{1}{S} \sum_{i=1}^S \{ \log q_\lambda(\theta^{(i)}(\lambda)) - \log \Phi(\theta^{(i)}(\lambda)) \}$, which is then minimized by gradient descent using any automatic differentiation tool.
By its form, ADVI naturally admits stochastic gradient noisy scores since the gradient of the objective has the score $s(\theta) = \nabla_\theta \log \pi(\theta) = \nabla_\theta \log \Phi(\theta)$ in its linear form. 
One known caveat of ELBO-based VI is that the reverse KL objective is prone to lead to mode collapsing, i.e., failing to capture all modes of the target density. The score-based VI in the next section can potentially address this issue.


\subsection{Background on Score-based VI
}\label{sec:smvi_bg}

The idea of score-based VI is to match the scores between the true and variational distributions, that is, $\nabla_\theta \log q_\lambda(\theta) \approx \nabla_\theta \log \pi(\theta)$. 
The traditional approaches which we dub \textbf{Fisher} and \textbf{Score} aim to minimize the $q_\lambda$-expected least squares objectives, more specifically,
\begin{align}
loss_{Fisher/Score} := \mathbb{E}_{q_\lambda(\theta)} || \nabla_\theta \log q_\lambda(\theta) - s(\theta) ||^2
\label{eq:score_fisher}
\end{align}
where the norm is Euclidean for \textbf{Fisher}~\citep{fisher} and the $\textrm{Cov}(q_\lambda)$-induced one for \textbf{Score}~\citep{score}.
The optimization of (\ref{eq:score_fisher}) can be done by Monte-Carlo 
reparametrized sampling similarly as ADVI.

Recently, Gaussian score-matching idea has been introduced, inspired by online incremental algorithms for classifier estimation~\citep{pa,cw}. The first attempt called GSM~\citep{gsmvi} formulates a sequential constrained optimization, where they incrementally update $\{\lambda_t\}_t$ by solving:
\begin{align}
\lambda_{t+1} = \arg\min_\lambda \textrm{KL}(q_{\lambda_t} || q_\lambda) \ \ \textrm{s.t.} \label{eq:gsm} 
\ \ \nabla_\theta \log q_\lambda(\theta_t^{(i)}) = s(\theta_t^{(i)}), \ \theta_t^{(i)}\!\sim\! q_{\lambda_t}(\theta), \ i\!=\!1,\dots,S 
\end{align}
The follow-up work Batch-and-Match (BaM)~\citep{bam} turns the constrained optimization of (\ref{eq:gsm}) into unconstrained Lagrangian form (with coefficients $\gamma_t$):
\begin{align}
\lambda_{t+1} = \arg\min_\lambda \ \textrm{KL}(q_{\lambda_t} || q_\lambda) \ + \label{eq:bam} 
\ \frac{\gamma_t}{S} \sum_{i=1}^S \big\vert\big\vert \nabla_\theta \log q_\lambda(\theta_t^{(i)}) \!-\! s(\theta_t^{(i)}) \big\vert\big\vert_{\textrm{Cov}(q_\lambda)}^2, \ \theta_t^{(i)}\!\sim\! q_{\lambda_t}(\theta) 
\end{align}
Both GSM and BaM are confined to Gaussian $q_\lambda$, and lead to closed-form solutions for each iteration $t$ using the Karush–Kuhn–Tucker conditions. However, they do not allow for stochastic-gradient noisy score $\hat{s}(\theta)$, not suitable for large-scale BNN variational inference. 



\section{Our Approach}\label{sec:main}



To address the drawbacks of previous score-based VI methods, while specifically targeting a large-scale score-based VI, we propose a novel proximal stochastic-gradient score-based VI algorithm that can effectively deal with noisy scores. We first consider the noise-free score case with reasonable theoretical convergence arguments, then deal with noisy unbiased score cases. We often drop $\lambda$ in notation $q_\lambda(\theta)$ for simplicity.

At iteration $t$, let $q_t(\theta)$ be our current estimate. We solve the following optimization problem whose solution becomes our next iterate $q_{t+1}(\theta)$. 
\begin{align}
&q_{t+1} \ = \ \arg\min_{q\in\mathcal{Q}} \ \mathbb{E}_{\theta\sim q_t(\theta)}\!\Big[
\alpha_t \big|\big| \nabla_\theta \log q(\theta) \! - \! \nabla_\theta \log q_t(\theta) \big|\big|_2^2 \ 
+ \ 
\big|\big| \nabla_\theta \log q(\theta) - s(\theta) \big|\big|_2^2
\Big]
\label{eq:prox_sm}
\end{align}
We can see that it is proximal score matching by enforcing the next iterate to remain close to the current one $q_t$ in scores. We have adaptive proximal impact $\alpha_t$, which initially is close to 0, but as iteration goes on, $\alpha_t$ increases and converges to some constant $\alpha_\infty \leq 1$. 
We use the linear scheduling $\alpha_t = \frac{t}{T}$ with the maximum number of iterations $T$. Clearly $\alpha_\infty \to 1$ as $T\to\infty$. 

The proximal term in (\ref{eq:prox_sm}) allows the expectation to be taken with respect to (constant) $q_t$ instead of (variable) $q$, which circumvents differentiating twice over $\theta$ unlike previous {\em Score} and {\em Fisher}, 
\begin{wrapfigure}[10]{r}{0.533\textwidth}
\centering
\begin{minipage}{0.533\columnwidth}
\vspace{-1.25em}
\begin{algorithm}[H]
\caption{Our 
Score-Matching 
VI Algorithm.
}
\label{alg:sgsmvi}
\begin{small}
\begin{algorithmic}
\STATE \textbf{Input:} $T=$ $\#$ iterations, $N=$ $\#$ steps per iteration, 
\INDSTATE[2] \ \ $S=$ $\#$  MC samples, $\alpha_t=t/T$, initial $q_0(\theta)$.
\STATE \textbf{For} $t=0, 1, \dots, T-1$:
    \INDSTATE[1] Sample $\theta^{(i)} \sim q_t(\theta)$, \ \ $i=1,\dots,S.$ 
    \INDSTATE[1] 
    Compute the noisy mini-batch score $\hat{s}^{(i)}$ at $\theta^{(i)}$.
    \INDSTATE[1] \textbf{For} step $=0, 1, \dots, N-1$:
        \INDSTATE[2] Update $\lambda \!\leftarrow\! \lambda \!-\! \eta \frac{d loss(\lambda)}{d \lambda}$ as (\ref{eq:noisy_loss_grad}) with $\theta^{(i)}$, $\hat{s}^{(i)}$.
    \INDSTATE[1] $q_{t+1} := q_\lambda$.
\end{algorithmic}
\end{small}
\end{algorithm}
\end{minipage}
\end{wrapfigure} 
and removes the need for reparametrized sampling unlike ADVI. Another thing to note is that we use the Euclidean norm instead of $\textrm{Cov}(q)$-induced norm.  BaM used the latter to ensure closed-form updates for Gaussian family. However, it might incur issues for certain initial choices of $q$. For instance, if $\textrm{Cov}(q)$ initially has eigenvalues close 0, we only have very slow update of $q$ along the corresponding eigenvector directions. This initial iterate sensitivity of BaM is empirically illustrated in Sec.~\ref{sec:expmt_gaussian} (Fig.~\ref{fig:p_gaussian} (Bottom row)).


\textbf{Convergence analysis. 
} 
We have the following 
theorem on the convergence of our method:
\theoremstyle{plain}
\begin{theorem}
Assuming that\footnote{Both assumptions can be thought to be ideal. Especially the second assumption may not be satisfied in practice depending on the chosen variational family $\mathcal{Q}$.
} 
i) the minimization in (\ref{eq:prox_sm}) is exact in each iteration, and ii) the variational family $\mathcal{Q}$ is rich enough to represent smooth vector fields $F:\mathbb{R}^d \to \mathbb{R}^d$, specifically for any $\theta \in \Theta$, $q\in\mathcal{Q}$, and $\beta\in[0,1]$, we have $q'\in\mathcal{Q}$ such that 
\begin{align}
\nabla_\theta \log q'(\theta) \simeq \beta \nabla_\theta \log q(\theta) + (1-\beta) s(\theta)
\label{eq:assumption2}
\end{align}
Then for $\alpha_t = \frac{t}{T}$ we have $q_T(\theta) \to \pi(\theta)$ as $T\to\infty$.
\label{appthm:convergence}
\end{theorem}
\begin{proof}
The second assumption may sound very strong, and it may not be satisfied by conventional parametric density families $\mathcal{Q}$ in practice. However, if the true posterior is Gaussian, and if $\mathcal{Q}$ is also a Gaussian family, then this assumption holds. To see this, let $\pi(\theta) = \mathcal{N}(\theta; \mu, \Sigma)$ and $q(\theta) = \mathcal{N}(\theta; m, V)$. Then (\ref{eq:assumption2}) holds for $q'(\theta) = \mathcal{N}(\theta; m', V')$ where
\begin{align}
m' = \big(\beta V^{-1} \!+\! (1\!-\!\beta)\Sigma^{-1}\big)^{-1} \big(\beta V^{-1} m \!+\! (1\!-\!\beta) \Sigma^{-1} \mu \big), \ \ \ \ 
V' = \big(\beta V^{-1} \!+\! (1\!-\!\beta)\Sigma^{-1}\big)^{-1}
\end{align}
Now we 
prove the theorem under the assumptions.
We first regard $\nabla_\theta \log q(\theta)$ as a functional variable denoted by $v(\theta)$. Then solving (\ref{eq:prox_sm}) corresponds to:
\begin{align}
\min_{v(\cdot)} \ \int q_t(\theta) \Big(
\alpha_t \big|\big| v(\theta) - \nabla_\theta \log q_t(\theta) \big|\big|_2^2 + 
\big|\big| v(\theta) - s(\theta) \big|\big|_2^2
\Big) d\theta
\label{eq:prox_sm_functional}
\end{align}
By taking the functional gradient with respect to $v(\cdot)$ and setting it to zero, where the latter has a unique solution due to the second assumption, 
the following equality holds for the optimal $q_{t+1}$:
\begin{align}
\nabla_\theta \log q_{t+1}(\theta) \simeq \frac{\alpha_t}{1+\alpha_t} \nabla_\theta \log q_t(\theta) + \frac{1}{1+\alpha_t} s(\theta)
\label{eq:score_recurrence}
\end{align}
Clearly $q_{t+1}\in\mathcal{Q}$ due to (\ref{eq:assumption2}). 
By telescoping the recurrence (\ref{eq:score_recurrence}), we can find the closed-form solution for the recurrence as:
\begin{align}
\nabla_\theta \log q_T(\theta) \simeq s(\theta) + \gamma_T \big( \nabla_\theta \log q_0(\theta) - s(\theta) \big), \ \ \ \ \gamma_T := \prod_{t=0}^{T-1} \frac{\alpha_t}{1+\alpha_t}
\end{align}
Since $\alpha_t \leq \alpha_\infty \leq 1$ for all $t$, we get $\gamma_T \to 0$ as $T\to\infty$. This proves that the score of $q_T(\theta)$ converges to the target score $s(\theta)$. Then we use the fact that score matching implies distribution matching (Lemma~\ref{lemma:sm=dm} in Appendix), which completes the proof. 
\end{proof}

\textbf{Convergence in noisy score cases.} 
If we replace the true score $s(\theta)$ in (\ref{eq:prox_sm}) by a noisy unbiased sample $\hat{s}(\theta)$, we have a noisy version of the loss that can be written as a function of $\lambda$ 
as follows\footnote{
The ultimate loss function or what the algorithm converges to is rather implicit since it is a proximal method. In this sense (\ref{eq:noisy_loss}) is not an ultimate loss function in the traditional sense.
}:
\begin{align}
loss(\lambda) := \mathbb{E}_{\theta\sim q_t(\theta)}\Big[
\alpha_t \big|\big| \nabla_\theta \log q_\lambda(\theta) - \nabla_\theta \log q_t(\theta) \big|\big|_2^2 
\ + \ \big|\big| \nabla_\theta \log q_\lambda(\theta) - \hat{s}(\theta) \big|\big|_2^2
\Big]
\label{eq:noisy_loss}
\end{align}
For each $\theta\sim q_t(\theta)$ which is constant in $\lambda$, the gradient of the loss with respect to $\lambda$ becomes: 
\begin{align}
\frac{d loss(\lambda)}{d \lambda} =   
(1\!+\!\alpha_t) \frac{\partial ||\nabla \log q_\lambda(\theta)||^2}{\partial\lambda} \!-\! \label{eq:noisy_loss_grad} 
2\alpha_t \frac{\partial \nabla \log q_\lambda(\theta)}{\partial \lambda} \nabla \log q_t(\theta) \!-\! 2 \frac{\partial \nabla \log q_\lambda(\theta)}{\partial \lambda} \hat{s}(\theta)
%
\end{align}
In this derivation we emphasize that the loss gradient in (\ref{eq:noisy_loss_grad}) is an {\em unbiased} estimate, i.e., the expectation of (\ref{eq:noisy_loss_grad}) over the randomness of $\hat{s}(\theta)$ becomes identical to the true loss gradient that can be attained with the true score $s(\theta)$, thanks to the linear appearance of the noisy $\hat{s}(\theta)$ term related to the variable $\lambda$ ($\theta$ is constant here). 
We re-emphasize that no reparametrization sampling trick is required, and the gradient with respect to $\lambda$ operates only on the density $q$'s internal part, not through $\theta$. 
Compared to ELBO-based ADVI, we have some efficiency gain in practice 
as demonstrated in Sec.~\ref{sec:complexity} (Table~\ref{tab:vision_time_memory}). 
Our overall algorithm is summarized as pseudocode in Alg.~\ref{alg:sgsmvi}.

\section{Related Work}\label{sec:related}



The literature on 
variational inference is massive, and the latest research ideas and directions can be found in some of the recent papers such as~\citep{bbvi2015,bbvi_ryder18,bbvi_locatello18,bbvi2019,bbvi_weladawe22,bbvi_burroni23,bbvi_domke23,bbvi_kim23,bbvi_giordano24}. We have already discussed score-based VI algorithms~\citep{fisher,score,gsmvi,bam} that are closely related to ours in Sec.~\ref{sec:setup}. Other score-based VI methods include among others the Fisher divergence optimization where in~\citep{sivi} they considered semi-implicit variational families to optimize the Fisher divergence, and in~\citep{evi} the energy-based variational families were dealt with. 

Our method is a proximal method that iteratively refines the estimate while gradually improving the score matching quality. Similar proximal ideas were previously adopted for variational inference, but in different flavors. For instance, one reasonable objective is an ELBO combined with a (forward) KL regularization term as studied in~\citep{proximal_point, khan15,khan16}. 
In~\citep{lambert22}, a similar proximal algorithm was proposed for the KL divergence objective with the Wasserstein metric as a regularizer.
In~\citep{diao23} they introduced a proximal-gradient method with a closed-form proximal step. 
Although we have discussed some reasonable theoretical arguments on the convergence of the proposed algorithm, the assumptions we made may not be satisfied in practice. 
More rigorous convergence analysis may require stochastic inexact proximal point methods~\citep{JMLR:v24:21-1219,ispp25} and related techniques. We only verified empirical convergence with supporting arguments in this paper, and more rigorous study will be done as future research.


\section{Experiments}\label{sec:experiments}


We test the performance of our proposed method in various variational inference settings including large-scale visual recognition with ViT BNNs and time-series forecasting with deep BNNs. We compare our method with the following competing VI methods: the ELBO-based VI ({\em ADVI})~\citep{advi}, two recent {\em GSM}~\citep{gsmvi}, and {\em BaM}~\citep{bam}. We exclude {\em Score}~\citep{score} and {\em Fisher}~\citep{fisher} since they require differentiating twice over $\theta$ (through neural networks), computationally demanding, especially large memory footprint for the large-scale Bayesian deep learning (e.g., ViT) experiments in Sec.~\ref{sec:vit}. But we provide comparisons with these models on synthetic data in Appendix. 
Our empirical study is performed on: 
synthetic Gaussian, Gaussian mixtures cases as proof of concept in Sec.~\ref{sec:toy}, the PosteriorDB Bayesian posterior inference benchmark~\citep{posteriordb} 
in Appendix~\ref{appsec:more_expmt}, and real-world large-scale visual recognition and time-series forecasting benchmarks in Sec.~\ref{sec:mnist_low_data},~\ref{sec:vit}, and \ref{sec:ts}. 

For ADVI and our method, we use the Adam optimizer~\citep{adam} unless stated otherwise, where the learning rates are chosen between $10^{-1}$ and $10^{-3}$ by some pilot runs with small number iterations. For all methods we randomly initialized the variational density $q$ unless stated otherwise, and use the same initial $q$ for all competing methods. Further training details and settings can be found in Appendix~\ref{appsec:expmt_details}. 
For the evaluation metric, we report empirical divergences between the variational and the target
distributions with respect to the number of score calls. We use the forward KL (FKL)\footnote{
In general, the forward KL, if possible to compute, is considered to be a better metric to measure the similarity between the true and approximate distributions than the backward KL. The backward KL is hard to detect mode collapsing, for instance, the mode collapsing cases of $q \! \ll \! \pi$ (i.e., $q$ is absolutely continuous with respect to $\pi$). This may be one drawback of the ELBO-based VI that aims to minimize the backward KL loss. 
}, i.e., $\textrm{KL}(\pi||q)$ as a metric for synthetic experiments (Sec.~\ref{sec:toy}) while the negative ELBO used for real-world experiments (Sec.~\ref{sec:mnist_low_data},~\ref{sec:vit}, and \ref{sec:ts}). 
Additional results can be found in Appendix~\ref{appsec:additional_expmts}.

\subsection{Synthetic Experiments}\label{sec:toy}

Before testing our method on BNN variational inference tasks, we 
study the convergence of the proposed approach on several synthetic data settings: Gaussian and non-Gaussian targets $\pi(\theta)$ where the state $\theta$ is not neural net parameters, but simply a random vector. 
We analyze whether or not the algorithms converge and how fast they converge with respect to the number of score function calls. 

\subsubsection{Gaussian Target Cases}\label{sec:expmt_gaussian}


\begin{wrapfigure}[18]{r}{0.51\textwidth}  
\centering
\vspace{-4.7em}
\begin{center}
\centering
\includegraphics[trim = 0mm 5mm 9mm 4mm, clip, scale=0.205 
]{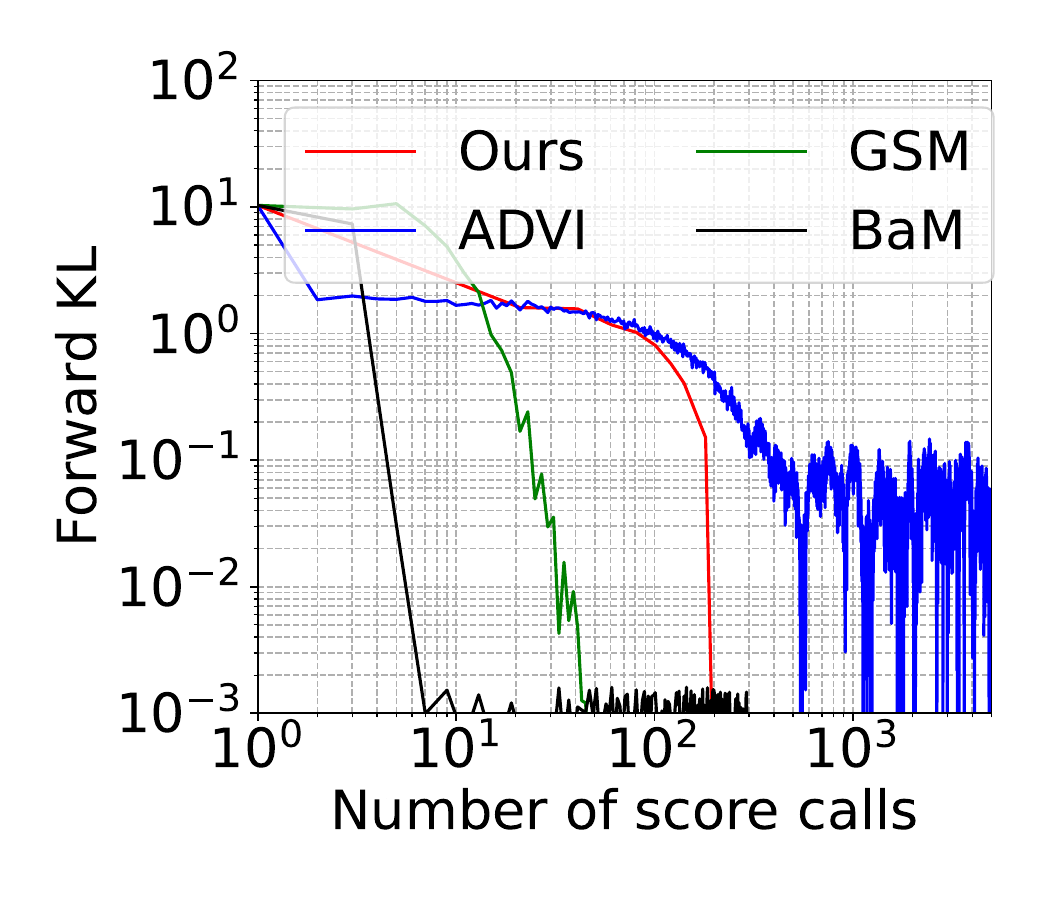}
%
\includegraphics[trim = 8.5mm 5mm 7mm 4mm, clip, scale=0.205 
]{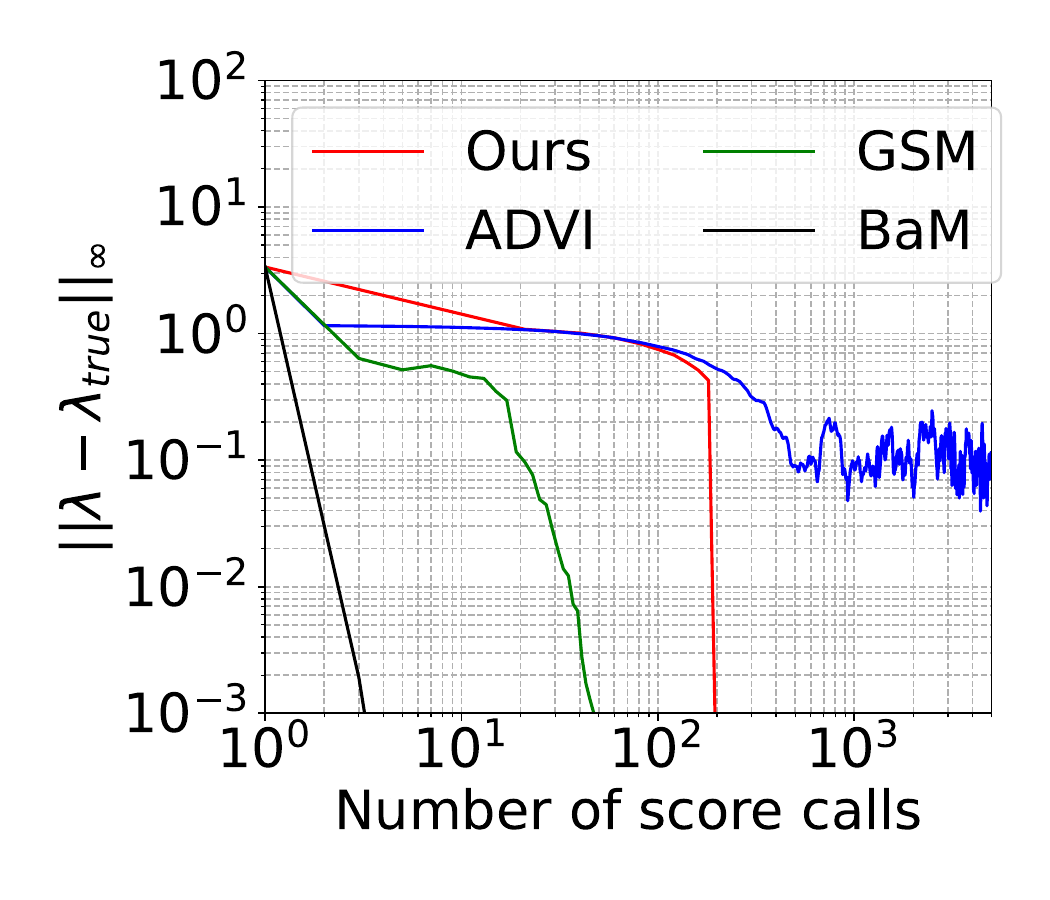}
%
\\
\includegraphics[trim = 5mm 5mm 9mm 4mm, clip, scale=0.205 
]{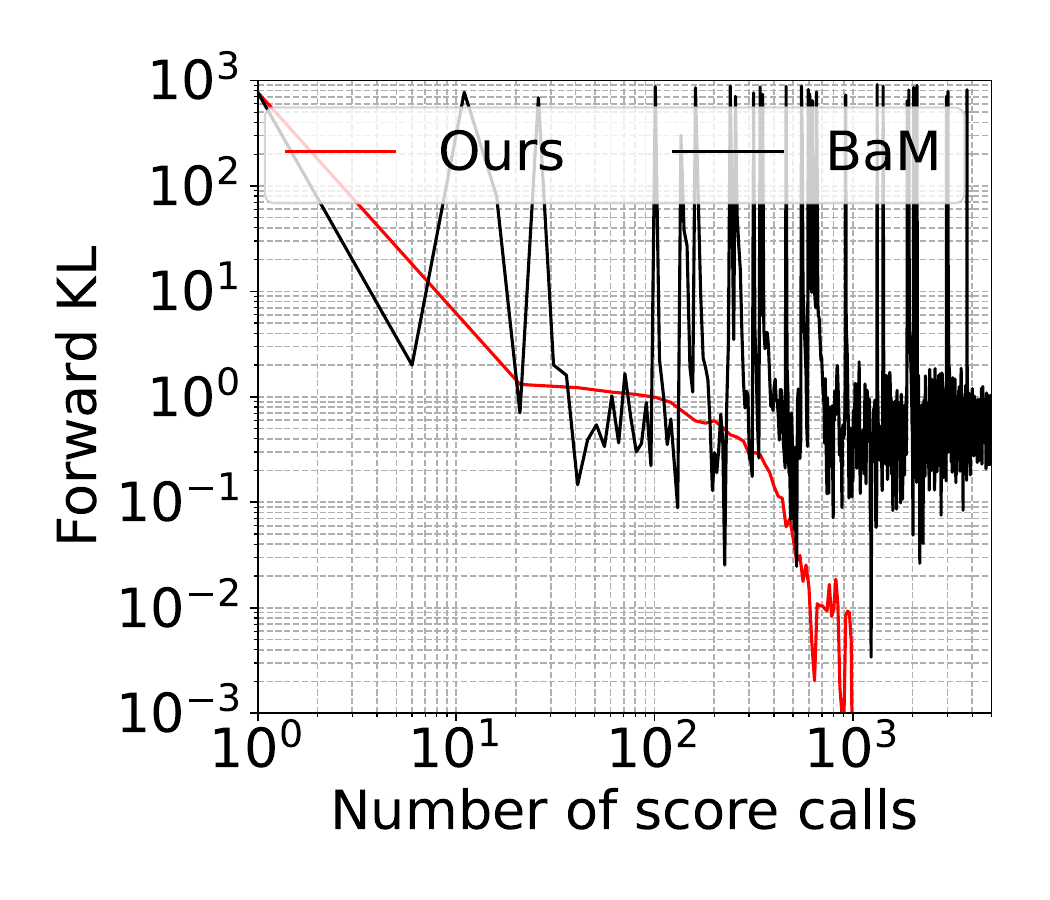}
%
\includegraphics[trim = 8.5mm 5mm 9mm 4mm, clip, scale=0.205 
]{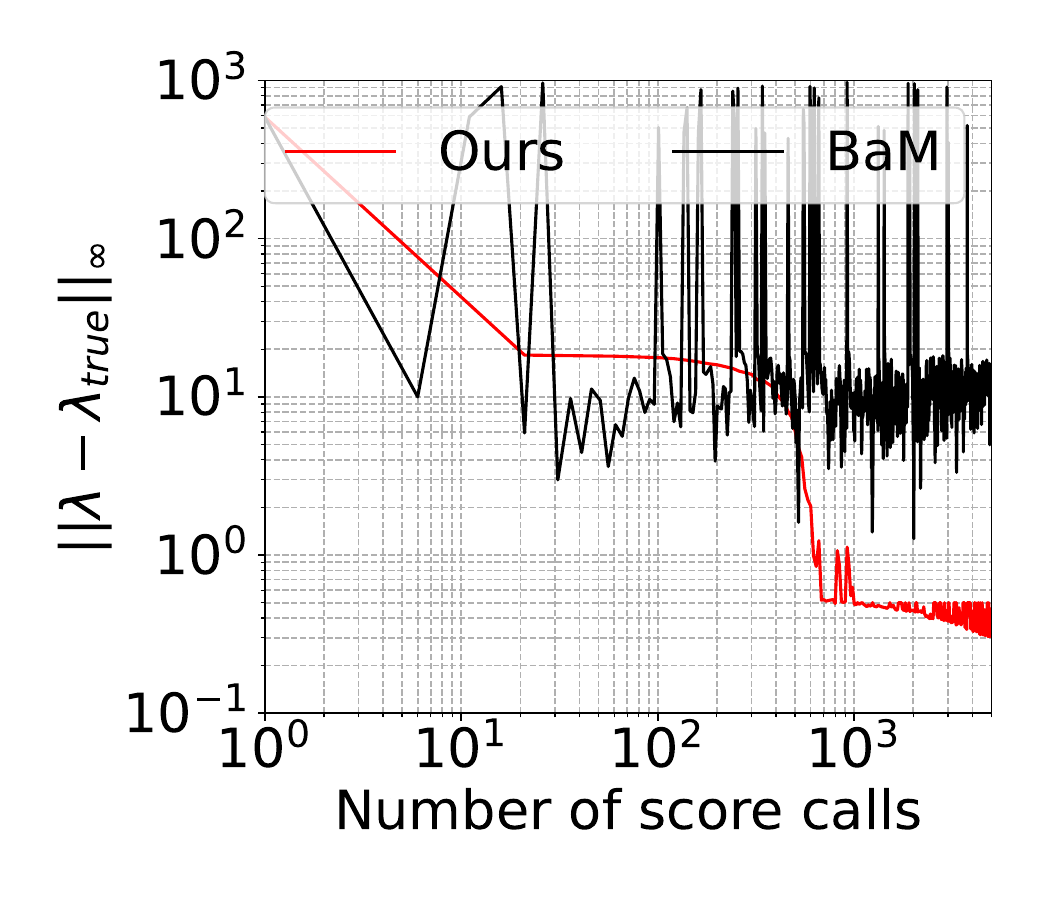}
%
%
\end{center}
\vspace{-1.5em}
\caption{
(Gaussian target cases) 
(Top row) Convergence as the number of score calls increases for forward KL (Left) and the error in Gaussian parameters $\lambda = (m,V)$ (Right). 
(Bottom row) Convergence for initial $q$ with small $\textrm{Cov}(q)$ eigenvalues. 
}
\label{fig:p_gaussian}
\end{wrapfigure}
We first consider the Gaussian $\pi(\theta)$ cases where we randomly generate parameters $\lambda_{true}=(m_{true},V_{true})$ as $\pi(\theta)$ with $\dim(\theta)=3$. The competing variational inference algorithms are compared in Fig.~\ref{fig:p_gaussian} (Top). 
We see that BaM and GSM especially work well with fastest convergence rates as they have closed-form updates under Gaussian assumption. 
Our approach converges in reasonable rates, 
however, ADVI suffers from slow convergence with much fluctuation. 
To see sensitivity to the initial $q$ choice regarding the $\textrm{Cov}(q)$-norm based methods, especially BaM, we start with $q$ that has very small covariance eigenvalues. As shown in 
Fig.~\ref{fig:p_gaussian} (Bottom), 
BaM struggles to converge since its ignorance of score mismatches along the small covariance eigenvector directions. 
In the diagonal covariance cases, for instance, those dimensions corresponding to small variances would not be optimized properly to match the scores of $q$ and $\pi$. Our approach is more robust to initial $q$ choice.


\subsubsection{
Non-Gaussian Target Cases 
}\label{sec:expmt_mog}
\vspace{-0.5em}
\begin{wrapfigure}[21]{r}{0.51\textwidth}  
\vspace{-4.5em}
\begin{center}
\centering
\includegraphics[trim = 3mm 18mm 280mm 1mm, clip, scale=0.180 
]{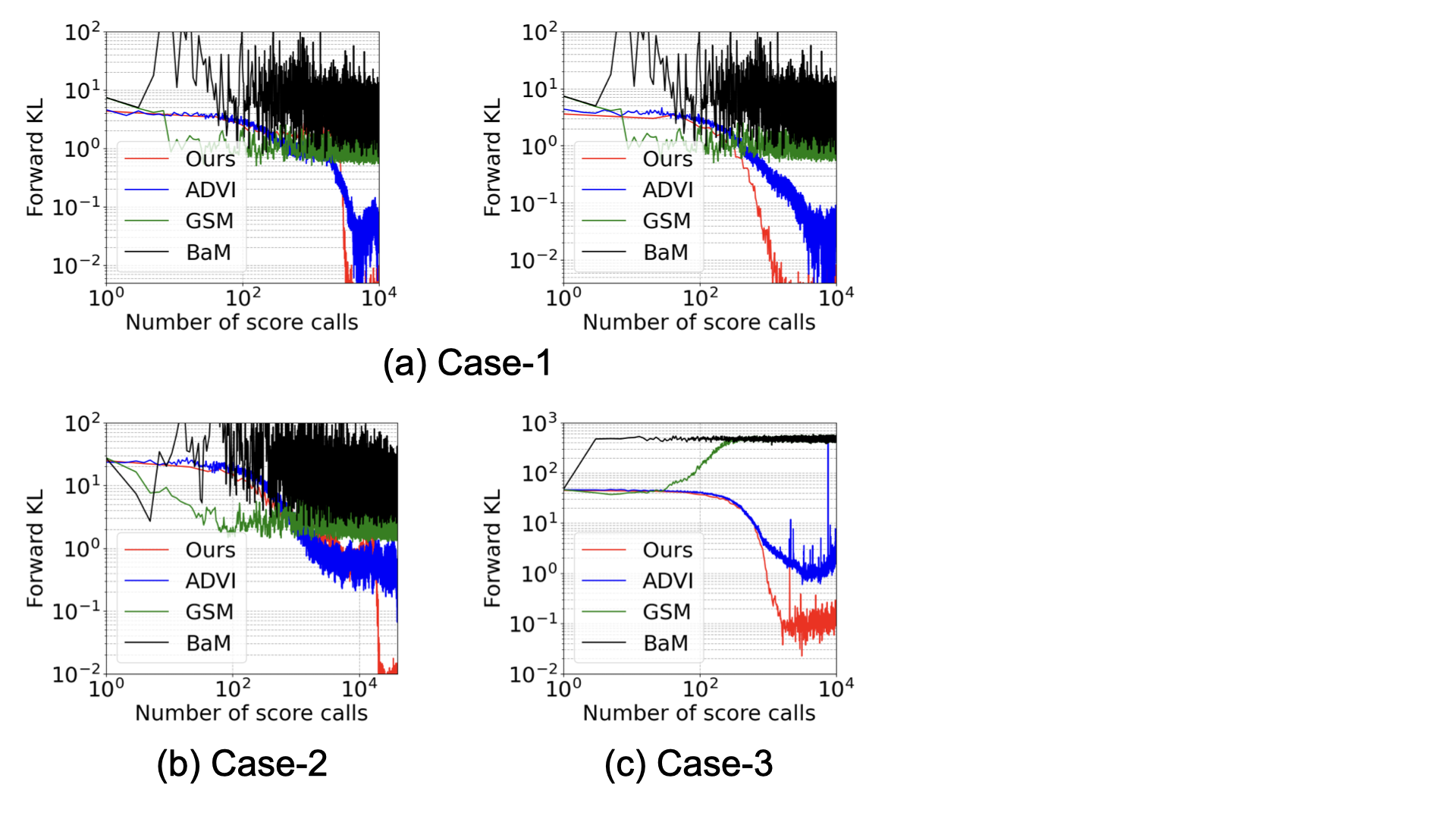}
\end{center}
\vspace{-1.0em}
\caption{
(Gaussian mixtures) 
(a) Case-1: Convergence in the order-2 Gaussian mixture target case. 
Ours and ADVI use an order-2 mixture of Gaussians $q$ (Left) and an order-3 mixture of Gaussians $q$ (i.e., model mismatch) (Right).
(b) Case-2: When the target $\pi$ has order 5. 
(c) Case-3: When $\dim(\theta)\!=\!30$. 
}
\label{fig:p_mog}
\end{wrapfigure}
Next we consider the multi-modal target cases where $\pi$ is mixtures of Gaussians. 
We study three experimental scenarios: i) identifiability of the target when the variational density $q_\lambda$ is a Gaussian mixture, ii) impact of the increased number of modes in the target $\pi$, and iii) impact of the high-dimensional ($\theta$) state space.

\textbf{Case-1 (Identifiability of Gaussian mixture target and robustness to model mismatch):} With $\dim(\theta)\!=\!3$ and order-2 mixture $\pi$, we first check convergence when the variational density family is the same order-2 mixture (i.e., model match). As shown in Fig.~\ref{fig:p_mog}(a) (Left), our method 
converges with the smallest forward KL score and the convergence rate faster than ADVI. GSM and BaM are not successful due to their Gaussian $q$ restriction. 
Next, when we choose $q$ to have an order-3 mixture (i.e., model mismatch\footnote{
This model mismatch setting ensures that the variational family is large enough to subsume the target distribution (e.g., by having 0 mixing proportion for the third component). Sec.~\ref{app_sec:model_mismatch2} in Appendix discusses the model mismatch scenario where the variational family doesn't include the true posterior.
}), as shown in Fig.~\ref{fig:p_mog}(a) (Right), our method still converges to the similar FKL value which is about 10 times lower than that of ADVI. When we inspect the learned $q$, we 
identify one component that is negligible with mixing proportion close to 0, demonstrating 
our approach is robust to model mismatch.
%

\textbf{Case-2 (Increased number of modes in the target):} 
When we increase the mixture order of $\pi$ to 5, thus more distinct from a single Gaussian, as shown in Fig.~\ref{fig:p_mog}(b), our method still works well, with the final FKL value even lower than ADVI. The inferior performance of ADVI 
can be attributed to the mode collapsing due to its reverse KL objective, and possibly to the Gumbel-softmax approximation error for reparametrized sampling with Gaussian mixtures.
%

\textbf{Case-3 (Increased state dimensionality):} 
When we increase the state dimensionality from 3 to 30, as shown in Fig.~\ref{fig:p_mog}(c), our method still works gracefully. ADVI yields a suboptimal solution while GSM and BaM fail to work with numerical issues.

\subsection{
MNIST Experiment
}\label{sec:mnist_low_data}


\begin{wrapfigure}[22]{r}{0.29\textwidth}  
\vspace{-5.5em}
\begin{center}
\centering
\includegraphics[trim = 5mm 5mm 6mm 4mm, clip, scale=0.245 
]{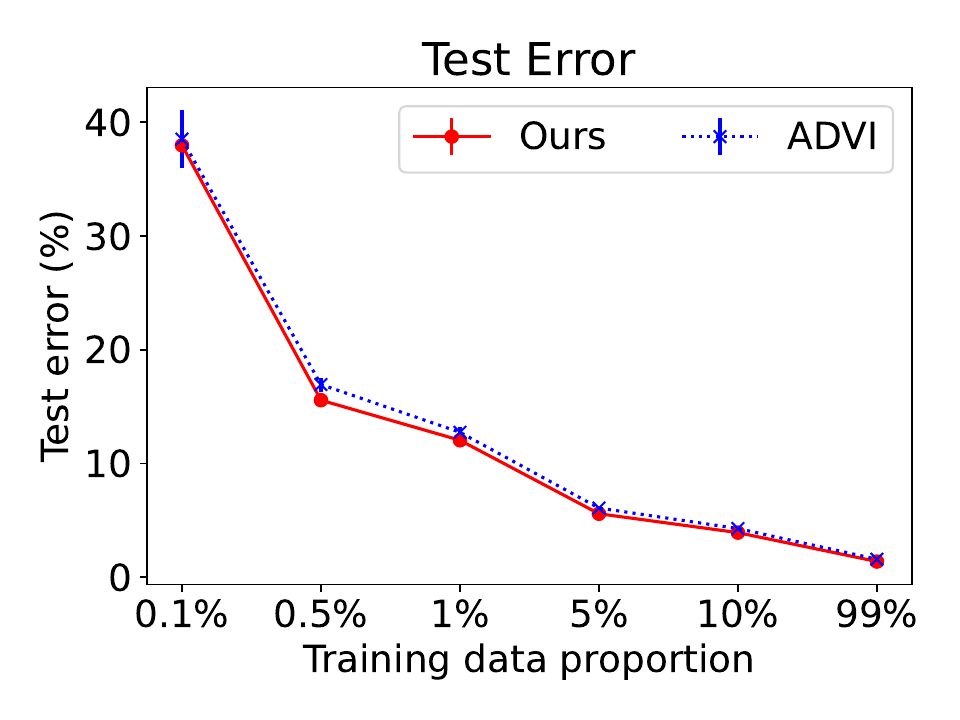} \\ 
\includegraphics[trim = 5mm 5mm 6mm 4mm, clip, scale=0.245 
]{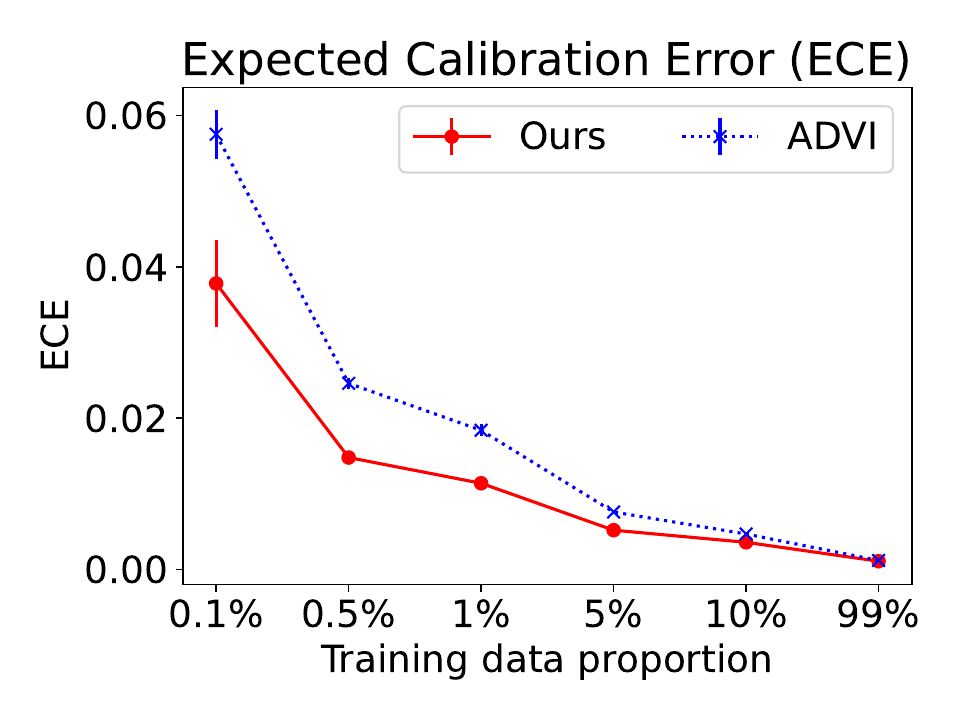} \\ 
\includegraphics[trim = 5mm 5mm 6mm 4mm, clip, scale=0.245 
]{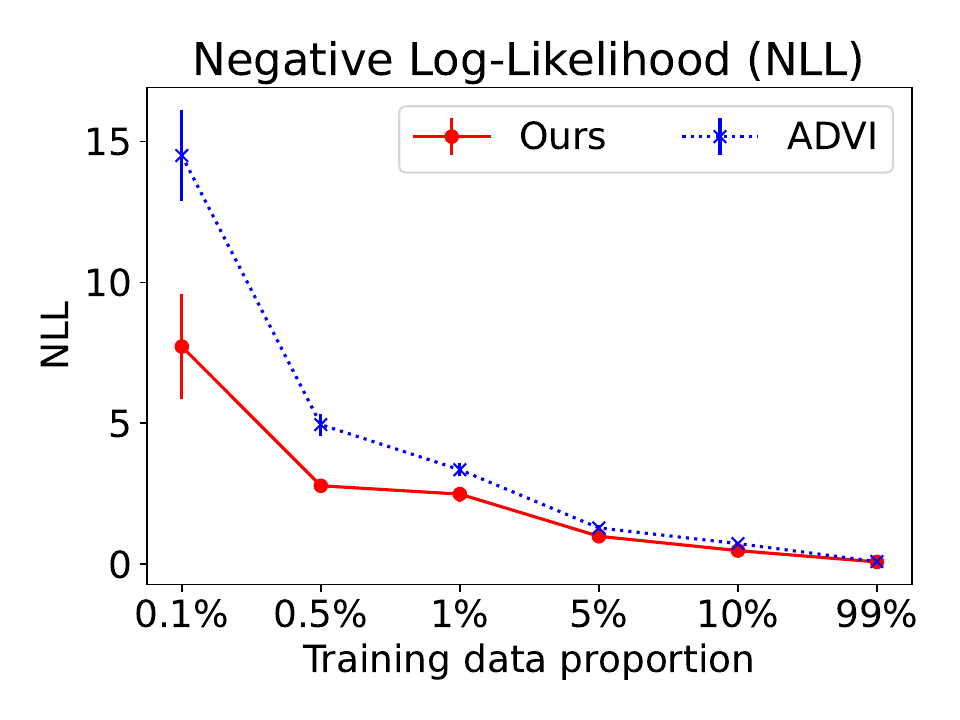}
\end{center}
\vspace{-1.4em}
\caption{
(MNIST) Impact of training data size on test prediction and uncertainty quantification.
}
\label{fig:mnist}
\end{wrapfigure}
We move on to the real-world Bayesian deep learning experiment with the MNIST dataset and a multi-layer perceptron (MLP) network. As the original MNIST training data split of 60K samples leads to the performance saturation\footnote{Our algorithm achieved $1\!\sim\!2\%$ test error rate as shown by the right-most (red) point in the top panel of Fig.~\ref{fig:mnist}.}, we focus on low training data regimes by reducing training data size substantially. We aim to see model's prediction performance and uncertainty quantification capability. To this end, we randomly split the original training data into $p\%$ training and $(100\!-\!p)\%$ validation sets with $p\in\{0.1, 0.5, 1, 5, 10, 99\}$. For instance, $p\!=\!0.1$ corresponds to the extreme low data regime of training data size 60 samples (6 images per class). We use the validation data to determine early stopping of training iterations. 

The variational family is Gaussian where $\lambda$ consists of mean and diagonal covariance for the MLP parameters $\theta$. 
We first observe that at this model scale (three hidden layers with 1000 units) GSM and BaM simply fail to work due to numerical issues. 
We focus on comparing our method and ADVI where both start with the same randomly initialized $\lambda$. 
In Fig.~\ref{fig:mnist} we plot test errors and 
the two uncertainty quantification metrics, the {\em expected calibration error (ECE)}~\citep{ece} and  
the test {\em negative log-likelihood (NLL)}. 
The former is computed by $\textrm{ECE} =  \frac{1}{M} \sum_{m=0}^{M-1} \frac{binsize[m]}{N\cdot K} \cdot \big| acc[m] - conf[m] \big|$ where we group the predictions $\{p(y\!=\!j|x)\}_{j=1}^K$ into $M$ equal-size bins, and count the cases of accurate prediction and the confidence values $p(y\!=\!j|x)$ in $acc[m]$ and $conf[m]$, respectively.
The latter is defined as: $\mathbb{E}_{(x^*,y^*)\sim D^*} [-\log p(y^*|x^*,D)]$ for the class predictive model $p(y^*|x^*, D)$ on test data $D^*$ where $p(y^*|x^*,D) = \int p(y^*|x^*,\theta) p(\theta|D) d\theta$ is the 
predictive distribution from the posterior estimate $p(\theta|D)$ on training data $D$.
As shown, while the test performance of both models are equally good, our approach achieves better uncertainty quantification than ADVI, which can be attributed to 
the proximal score matching objective.

\subsection{Large-scale BNNs for 
Visual Recognition
}\label{sec:vit}

\begin{table*}[t!]
\setlength{\tabcolsep}{2.85pt}
\vspace{-2.0em}
\caption{
Large-scale Bayesian deep learning experimental results.
}
\centering
\begin{scriptsize}
\centering
\begin{tabular}{cccccccc}
\toprule
& & 
\multicolumn{2}{c}{Error ($\%$) $\downarrow$} &
\multicolumn{2}{c}{ECE $\downarrow$} &
\multicolumn{2}{c}{NLL $\downarrow$} \\
\cmidrule(lr){3-4} \cmidrule(lr){5-6} \cmidrule(lr){7-8} 
\multicolumn{2}{c}{\multirow{2}{*}{}} & 
ResNet101 & ViT-L-32 &
ResNet101 & ViT-L-32 &
ResNet101 & ViT-L-32 \\
\midrule
\multirow{6}{*}{Pets} & Vanilla\Tstrut & $9.6400^{\pm 0.2688}$ & $8.1400^{\pm 0.2909}$ & N/A & N/A & N/A & N/A \\
& ADVI\Tstrut & $9.1340^{\pm 0.2648}$ & $7.9100^{\pm 0.2327}$ & $0.0800^{\pm 0.0179}$ & $0.0600^{\pm 0.0063}$ & $0.3100^{\pm 0.0057}$ & $0.2518^{\pm 0.0047}$ \\
& MC-Dropout\Tstrut & ${9.4827^{\pm 0.1235}}$ & ${8.3155^{\pm 0.1625}}$ & ${0.0960^{\pm 0.0115}}$ & ${0.0539^{\pm 0.0068}}$ & ${0.3163^{\pm 0.0043}}$ & ${0.2569^{\pm 0.0012}}$ \\
& SGLD\Tstrut & $9.3767^{\pm 0.1239}$ & $8.6467^{\pm 0.0834}$ & $0.0473^{\pm 0.0040}$ & $0.0425^{\pm 0.0083}$ & $0.3052^{\pm 0.0070}$ & $0.2649^{\pm 0.0037}$ \\
& Lap.~Approx.\Tstrut & $10.0833^{\pm 0.1173}$ & $8.5767^{\pm 0.2243}$ & $0.0567^{\pm 0.0741}$ & $0.0475^{\pm 0.0083}$ & $0.3268^{\pm 0.0098}$ & $0.2684^{\pm 0.0011}$ \\
& Ours\Tstrut & $\pmb{9.1140^{\pm 0.1822}}$ & $\pmb{7.7460^{\pm 0.2403}}$ & $\pmb{0.0440^{\pm 0.0049}}$ & $\pmb{0.0360^{\pm 0.0080}}$ & $\pmb{0.2986^{\pm 0.0050}}$ & $\pmb{0.2468^{\pm 0.0037}}$ \\
\midrule
\multirow{6}{*}{Flowers} & Vanilla\Tstrut & $15.0167^{\pm 0.2623}$ & $12.5100^{\pm 0.5535}$ & N/A & N/A & N/A & N/A \\
& ADVI\Tstrut & $12.6780^{\pm 0.8668}$ & $12.6560^{\pm 0.5455}$ & $0.0700^{\pm 0.0089}$ & $0.1580^{\pm 0.0075}$ & $0.5104^{\pm 0.0370}$ & $0.5116^{\pm 0.0213}$ \\

& MC-Dropout\Tstrut & $14.0333^{\pm 0.4620}$ & $13.2133^{\pm 0.2136}$ & $0.0594^{\pm 0.0039}$ & $0.0140^{\pm 0.0029}$ & $0.5500^{\pm 0.0283}$ & $0.5567^{\pm 0.0249}$ \\
& SGLD\Tstrut & $13.4867^{\pm 0.1855}$ & $12.8667^{\pm 0.1429}$ & $0.0647^{\pm 0.0101}$ & $0.1582^{\pm 0.0081}$ & $0.5195^{\pm 0.0093}$ & $0.5591^{\pm 0.0103}$ \\
& Lap.~Approx.\Tstrut & $13.7725^{\pm 0.5796}$ & $13.5212^{\pm 0.3906}$ & $0.0883^{\pm 0.0052}$ & $0.0367^{\pm 0.0111}$ & $0.5859^{\pm 0.0452}$ & $0.6477^{\pm 0.0029}$ \\
& Ours\Tstrut & $\pmb{12.0400^{\pm 1.1195}}$ & $\pmb{11.9700^{\pm 0.6237}}$ & $\pmb{0.0420^{\pm 0.0075}}$ & $\pmb{0.0080^{\pm 0.0098}}$ & $\pmb{0.4683^{\pm 0.0397}}$ & $\pmb{0.4930^{\pm 0.0184}}$ \\
\midrule
\multirow{6}{*}{Aircraft} & Vanilla\Tstrut & $48.0100^{\pm 0.5575}$ & $54.3000^{\pm 0.4864}$ & N/A & N/A & N/A & N/A \\
& ADVI\Tstrut & $45.4200^{\pm 0.3993}$ & $53.0020^{\pm 1.3731}$ & $0.0440^{\pm 0.0049}$ & $0.0620^{\pm 0.0040}$ & $1.8084^{\pm 0.0216}$ & $2.0754^{\pm 0.0649}$ \\

& MC-Dropout\Tstrut & $45.9700^{\pm 1.2865}$ & $53.2133^{\pm 0.5000}$ & $0.0237^{\pm 0.0017}$ & $0.0757^{\pm 0.0046}$ & $1.7613^{\pm 0.0092}$ & $\pmb{2.0400^{\pm 0.0096}}$ \\
& SGLD\Tstrut & $47.5140^{\pm 0.8574}$ & $52.9125^{\pm 0.2943}$ & $0.0217^{\pm 0.0026}$ & $0.0603^{\pm 0.0074}$ & $1.7883^{\pm 0.0505}$ & $2.0883^{\pm 0.0102}$ \\
& Lap.~Approx.\Tstrut & $46.0425^{\pm 0.6013}$ & $53.5600^{\pm 0.1122}$ & $0.0427^{\pm 0.0020}$ & $0.0710^{\pm 0.0078}$ & $1.8162^{\pm 0.0062}$ & $2.0955^{\pm 0.0123}$ \\
& Ours\Tstrut & $\pmb{44.1960^{\pm 0.8833}}$ & $\pmb{52.3180^{\pm 0.4346}}$ & $\pmb{0.0180^{\pm 0.0098}}$ & $\pmb{0.0520^{\pm 0.0098}}$ & $\pmb{1.7280^{\pm 0.0245}}$ & ${2.0531^{\pm 0.0180}}$ \\
\midrule
\multirow{6}{*}{DTD} & Vanilla\Tstrut & $38.2100^{\pm 0.6766}$ & $34.0767^{\pm 0.2370}$ & N/A & N/A & N/A & N/A \\
& ADVI\Tstrut & $35.0440^{\pm 0.7463}$ & $32.8300^{\pm 0.6549}$ & $0.3520^{\pm 0.0640}$ & $0.1900^{\pm 0.0228}$ & $1.6586^{\pm 0.1853}$ & $1.2194^{\pm 0.0086}$ \\

& MC-Dropout\Tstrut & $35.8867^{\pm 0.5327}$ & $33.8267^{\pm 0.5331}$ & $0.3867^{\pm 0.0330}$ & $0.1675^{\pm 0.0148}$ & $\pmb{1.3642^{\pm 0.0166}}$ & $1.2337^{\pm 0.0200}$ \\
& SGLD\Tstrut & $36.0600^{\pm 0.2205}$ & $33.3700^{\pm 0.5588}$ & $0.3067^{\pm 0.0464}$ & $0.1875^{\pm 0.0426}$ & $1.5099^{\pm 0.0472}$ & $1.2498^{\pm 0.0203}$ \\
& Lap.~Approx.\Tstrut & $36.2933^{\pm 0.4666}$ & $33.8567^{\pm 0.1406}$ & $0.7500^{\pm 0.0589}$ & $0.2625^{\pm 0.0238}$ & $1.8828^{\pm 0.0499}$ & $1.3343^{\pm 0.0232}$ \\
& Ours\Tstrut & $\pmb{34.4900^{\pm 0.5438}}$ & $\pmb{32.4980^{\pm 0.4829}}$ & $\pmb{0.1120^{\pm 0.0075}}$ & $\pmb{0.0900^{\pm 0.0210}}$ & ${1.4001^{\pm 0.0194}}$ & $\pmb{1.1936^{\pm 0.0092}}$ \\
\bottomrule
\vspace{-0.5em}
\end{tabular}
\end{scriptsize}
\label{tab:vision}
\vspace{-2.0em}
\end{table*}


Next we test our method on the large-scale backbone networks: ResNet-101~\citep{resnet} and Vision Transformers (ViT-L-32)~\citep{vit}, where the former consists of $\dim(\theta)\approx 43$ million parameters and the latter about $\dim(\theta)\approx 305$ million parameters. On the image classification vision tasks with datasets Pets~\citep{pets}, Flowers~\citep{flowers}, Aircraft~\citep{aircraft}, and DTD~\citep{dtd}, 
we randomly split the official training data into $50\%$ training and $50\%$ validation sets. 
As training such large networks from the scratch is very difficult, we take the publicly available pre-trained model weights in the form of Gaussian prior mean parameters in the Bayesian models, so that the variational density can find its support in the vicinity of the pre-trained weights~\citep{bayesdll}. We use the diagonal-covariance Gaussian for the variational density.

The results of test errors and uncertainty quantification are summarized in Table~\ref{tab:vision}. 
Note that GSM and BaM failed due to numerical explosion, and the traditional score-matching methods~\citep{score,fisher} at this scale quickly incurred out-of-memory error due to their computational overhead in differentiating twice over $\theta$ through neural networks.
We see that both ADVI and our method attain similar test prediction performance for both networks, but our method consistently outperforms ADVI in ECE and NLL uncertainty quantification metrics. This signifies better convergence 
with our score-based objective. 
For comparison with other approximate Bayesian inference methods, in Table~\ref{tab:vision} we also consider: the MC-dropout method~\citep{mc_dropout}, the stochastic gradient Langevin dynamics (SGLD)~\citep{sgld}, and the Laplace approximation. While the MC-dropout can be seen as variational inference with a special variational density family of mixtures of two spiky Gaussians, the other two methods are not variational inference methods. The results show that our approach still outperforms these inference methods consistently for both networks in both prediction errors and uncertainty quantification.

We also visualize the convergence of the our method and ADVI in Fig.~\ref{app_fig:vit_convergence} in Appendix. We show the test loss convergence instead of the negative ELBO on training data since the ELBO is merely a bound on the KL divergence and as the dimensionality grows at this model scale, it may not be a good metric to capture distribution divergence effectively. We see that our algorithm consistently converges faster to lower test loss values than ADVI.

\subsection{Time-series Forecasting 
}\label{sec:ts}

\begin{wraptable}[22]{r}{0.589\textwidth}
\setlength{\tabcolsep}{4.15pt}
\vspace{-4.85em}
\caption{Bayesian Koopa with Gaussian $q_\lambda(\theta)$. 
}
\centering
\begin{scriptsize}
\centering
\begin{tabular}{ccccc}
\toprule
 & Metric & Vanilla & ADVI & Ours \\ 
\midrule
\multirow{3}{*}{\rotatebox{90}{ECL\ \ \ \ }}\Tstrut & MSE $\downarrow$ & $0.1281 \pm 0.0002$ & $0.1231 \pm 0.0182$ & $\pmb{0.1228} \pm 0.0181$ \\
\cmidrule(lr){2-5} 
 & MAE $\downarrow$ & $\pmb{0.2303} \pm 0.0002$ & $0.2376 \pm 0.0198$ & $0.2374 \pm 0.0199$ \\
\cmidrule(lr){2-5} 
 & NLL $\downarrow$ & N/A & $0.5245 \pm 0.0524$ & $\pmb{0.5232} \pm 0.0513$ \\
\midrule
\multirow{3}{*}{\rotatebox{90}{ETTh2\ \ \ }} & MSE $\downarrow$ & $0.2413 \pm 0.0043$ & $0.2415 \pm 0.0048$ & $\pmb{0.2375} \pm 0.0028$ \\
\cmidrule(lr){2-5} 
 & MAE $\downarrow$ & $0.3101 \pm 0.0033$ & $0.3116 \pm 0.0048$ & $\pmb{0.3079} \pm 0.0038$ \\
\cmidrule(lr){2-5} 
 & NLL $\downarrow$ & N/A & $0.7239 \pm 0.0117$ & $\pmb{0.7145} \pm 0.0120$ \\
\midrule
\multirow{3}{*}{\rotatebox{90}{Ex.~Rate\ \ \ }} & MSE $\downarrow$ & $0.0475 \pm 0.0025$ & $0.0469 \pm 0.0014$ & $\pmb{0.0455} \pm 0.0014$ \\
\cmidrule(lr){2-5} 
 & MAE $\downarrow$ & $0.1523 \pm 0.0041$ & $0.1517 \pm 0.0027$ & $\pmb{0.1491} \pm 0.0026$ \\
\cmidrule(lr){2-5} 
 & NLL $\downarrow$ & N/A & $0.0084 \pm 0.0071$ & $\pmb{0.0015} \pm 0.0018$ \\
\midrule
\multirow{3}{*}{\rotatebox{90}{ILI\ \ \ \ \ }} & MSE $\downarrow$ & $2.2084 \pm 0.1340$ & $2.2233 \pm 0.0648$ & $\pmb{2.0764} \pm 0.0250$ \\
\cmidrule(lr){2-5} 
 & MAE $\downarrow$ & $0.9180 \pm 0.0292$ & $0.9299 \pm 0.0388$ & $\pmb{0.8862} \pm 0.0109$ \\
\cmidrule(lr){2-5} 
 & NLL $\downarrow$ & N/A & $1.8258 \pm 0.0082$ & $\pmb{1.7891} \pm 0.0063$ \\
\midrule
\multirow{3}{*}{\rotatebox{90}{Traffic\ \ \ \ \ }} & MSE $\downarrow$ & $0.4359 \pm 0.0002$ & $0.4302 \pm 0.0003$ & $\pmb{0.4290} \pm 0.0003$ \\
\cmidrule(lr){2-5} 
 & MAE $\downarrow$ & $0.2897 \pm 0.0012$ & $0.2870 \pm 0.0007$ & $\pmb{0.2857} \pm 0.0003$ \\
\cmidrule(lr){2-5} 
 & NLL $\downarrow$ & N/A & $1.0226 \pm 0.0003$ & $\pmb{1.0214} \pm 0.0003$ \\
\midrule
\multirow{3}{*}{\rotatebox{90}{Weather\ \ \ }} & MSE $\downarrow$ & $\pmb{0.1255} \pm 0.0006$ & $0.1273 \pm 0.0006$ & $0.1265 \pm 0.0003$ \\
\cmidrule(lr){2-5} 
 & MAE $\downarrow$ & $\pmb{0.1670} \pm 0.0008$ & $0.1698 \pm 0.0006$ & $0.1690 \pm 0.0003$ \\
\cmidrule(lr){2-5} 
 & NLL $\downarrow$ & N/A & $0.4152 \pm 0.0041$ & $\pmb{0.4103} \pm 0.0015$ \\
\bottomrule
\vspace{-0.5em}
\end{tabular}
\end{scriptsize}
\label{tab:ts_gaussian}
\end{wraptable}
We next test our method on the time-series forecasting problem that has strong real-world impacts. We specifically adopt the recent compute/memory-efficient  (deep) neural network called Koopa~\citep{koopa} that is based on the Fourier filter and the function space dynamic predictor based on the Koopman theory~\citep{koopman}. We turn the Koopa network into a BNN where the number of network parameters $\dim(\theta)$ ranges from 0.2M to 22M. We consider the Gaussian variational family (Sec.~\ref{sec:ts_nf} for the normalizing-flow family).
Following the experimental protocol from~\citep{koopa}, we use the six real-world time-series benchmark datasets, and the forecasting window size is $48$ (24 for ILI) with the covariate window size 96 (48 for ILI). 
Details of datasets and the experimental settings are in Appendix~\ref{appsec:expmt_details}. 
The results in Table~\ref{tab:ts_gaussian} show that our approach achieves (marginal) improvements over the non-Bayesian (vanilla) Koopa net, and consistently outperforms the ELBO-based VI. The smaller test negative log-likelihood scores, 
computed by 5 posterior samples with Gaussian mixture fitting, also indicate that our approach yields better uncertainty quantification.

\subsection{Normalizing Flow Variational Density}\label{sec:ts_nf}
\vspace{-0.8em}
Beyond Gaussian, we also consider the normalizing flow (NF)~\citep{pmlr-v37-rezende15} model for the variational density. More specifically, $\theta\sim q_\lambda(\theta)$ is defined as: $\theta = T_\lambda(z)$ for $z\sim \mathcal{N}(0,I)$ where $T_\lambda(z)$ is a flow neural network model with parameters $\lambda$. By the change of variables, we get:
\begin{align}
\log q_\lambda(\theta) = 
\log \mathcal{N}(T_\lambda^{-1}(\theta); 0, I) + \log |\det \nabla_\theta T_\lambda^{-1}(\theta)|
\end{align}
\begin{wraptable}[10]{r}{0.43\textwidth}
\setlength{\tabcolsep}{4.5pt}
\vspace{-1.95em}
\caption{
Bayesian Koopa with 
NF $q_\lambda(\theta)$. 
}
\centering
\begin{scriptsize}
\centering
\begin{tabular}{cccc}
\toprule
 & Metric & ADVI & Ours \\ 
\midrule
\multirow{2}{*}{\rotatebox{90}{Ex.~Rate \ \ }}\Tstrut & MSE $\downarrow$ & $0.3593 \pm 0.0271$ & $\pmb{0.2922} \pm 0.0162$ \\
\cmidrule(lr){2-4} 
 & MAE $\downarrow$ & $0.3294 \pm 0.0059$ & $\pmb{0.3125} \pm 0.0028$ \\
\cmidrule(lr){2-4} 
 & NLL $\downarrow$ & $0.7636 \pm 0.0107$ & $\pmb{0.7516} \pm 0.0090$ \\
\midrule
\multirow{3}{*}{\rotatebox{90}{Weather\ \ \ }} & MSE $\downarrow$ & $0.7965 \pm 0.0071$ & $\pmb{0.7695} \pm 0.0470$ \\
\cmidrule(lr){2-4} 
 & MAE $\downarrow$ & $0.3924 \pm 0.0066$ & $\pmb{0.3904} \pm 0.0042$ \\
\cmidrule(lr){2-4} 
 & NLL $\downarrow$ & $1.3249 \pm 0.0070$ & $\pmb{1.3171} \pm 0.0061$ \\
\bottomrule
\vspace{-0.5em}
\end{tabular}
\end{scriptsize}
\label{tab:ts_nf}
\vspace{1.0em}
\end{wraptable}
We test it on the time-series data from Sec.~\ref{sec:ts} with a modified Koopa network. Note that this is a very challenging application for NF since the state space is the neural net parameter space, and NF models usually suffer from high-dim state space. To this end, the dynamic and hidden dimensions of the Koopa network are parsimoniously reduced to yield $\dim(\theta)<5000$ so that numerical issues in high-dimensional state NF can be circumvented. Then we use the planar flow~\citep{pmlr-v37-rezende15} with a single layer. 
Results on two datasets \texttt{exchange-rate} and \texttt{weather} with forecasting window size $48$ are summarized in Table~\ref{tab:ts_nf}. Our score-based VI method consistently outperforms ADVI.

\subsection{Ablation Study}\label{sec:ablation}
\vspace{-0.5em}
\textbf{Impact of the proximal term.}
To see the impact of the proximal term in our method, we turn off the proximal term by setting the impact hyperparameter $\alpha_t\!=\!0$ for all $t$. For the order-5 Gaussian mixture target case (corresponding to Fig.~\ref{fig:p_mog}(b)), we have observed that the forward KL scores at the last iteration are: 0.0153 (with $\alpha_t\!=\!t/T$) and 0.4509 (with $\alpha_t\!=\!0$). 
This signifies the importance of the proximal term. 

\textbf{Impact of the number of MC samples $S$.}
Beyond default $S\!=\!1$, we also increase the number of Monte Carlo samples $S$ in estimating the expectation in the loss. For the Gaussian mixture target case (corresponding to Fig.~\ref{fig:p_mog}(b)), we tested our model with $S=5,10,20$. The iteration numbers that first hit the forward KL score less than $0.05$ are: 946 (default $S\!=\!1$), 96 ($S\!=\!5$), 89 ($S\!=\!10$), and 81 ($S\!=\!20$). So clearly increasing $S$ certainly helps expedite convergence. However, increasing $S$ greater than 1 may not be practical for large-scale BNNs due to computational overhead.

\subsection{Running Time and Memory Footprint}\label{sec:complexity}
\vspace{-0.5em}
\setlength{\intextsep}{4.0pt}%
\setlength{\columnsep}{4.0pt}%
\begin{wraptable}[6]{r}{0.41\textwidth}
\setlength{\tabcolsep}{2.75pt}
\vspace{-3.1em}
\centering
\caption{
Wall clock time and memory. 
}
\centering
\begin{scriptsize}
\centering
\begin{tabular}{ccccc}
\toprule
\multirow{2}{*}{Pets} & 
\multicolumn{2}{c}{Time (second)} &
\multicolumn{2}{c}{GPU memory (GB)} \\
\cmidrule(lr){2-3} \cmidrule(lr){4-5} 
& 
ResNet101 & ViT-L-32 &
ResNet101 & ViT-L-32 \\
\midrule
Vanilla\Tstrut & $0.1263$ & $0.2034$ & $0.30$ & $2.09$ \\
\midrule
ADVI\Tstrut & $0.2463$ & $0.3007$ & $0.96$ & $6.85$ \\
\midrule
%
Ours\Tstrut & $0.3256$ & $0.4323$ & $0.33$ & $2.30$ \\
\bottomrule
\vspace{-1.0em}
\end{tabular}
\end{scriptsize}
\label{tab:vision_time_memory}
\end{wraptable}
We do actual computational cost comparison between our method and ADVI. In Table~\ref{tab:vision_time_memory} we show per-iteration wall clock time and GPU memory footprints 
on Pets dataset. We use batch size 16 
on a single Nvidia Tesla V100 GPU. 
Although our method needs to maintain both previous $q_t$ and current $q$, 
and run multiple inner iteration steps, 
we have 
comparable computational cost 
compared to ELBO-based ADVI. 

\subsection{Additional Experiments}
\vspace{-0.6em}
In Sec.~\ref{appsec:additional_expmts} we have additional experimental results including: running time vs.~number of MC samples, statistical significance tests, and comparison with Fisher and Score methods.

\section{Conclusion}\label{sec:conclusion}
\vspace{-1.0em}
We have introduced a novel score-matching variational inference algorithm that can be effectively applicable to large-scale Bayesian neural network posterior inference. 
By combining the score-matching loss and the proximal penalty term in iterations, we effectively avoid the reparametrized sampling trick required in ELBO-based and traditional score-based methods, which makes our algorithm more flexible and scalable. 
Unlike the GSM methods that are confined to Gaussian variational density and noise-free scores, our algorithm allows for noisy unbiased mini-batch scores through stochastic gradients. Consequently, as demonstrated in our empirical study, the proposed algorithm is scalable to large-scale problems including Bayesian deep models with Vision Transformers and the time-series forecasting benchmarks with domain-tailored deep networks. 
On various real-world benchmarks, we have demonstrated the effectiveness of our algorithm in terms of faster convergence, 
high flexibility, and scalability in contrast to  existing score-based methods limited to small-sized problems.

\clearpage


{
\bibliographystyle{plainnat}
\bibliography{main}
}







\clearpage

\appendix

\centerline{\huge\textbf{{Appendix}}}



\section{Lemma~\ref{lemma:sm=dm} for Convergence Analysis}\label{appsec:lemma}

\begin{lemma}[Score matching implies distribution matching]
Assume that some regularity conditions hold for two distributions $p(\theta)$ and $q(\theta)$ (e.g., smooth distributions and Lipschitz continuous scores). 
We have $p = q$ if $\nabla_\theta \log p(\theta) = \nabla_\theta \log q(\theta)$ for all $\theta$ almost surely.
\label{lemma:sm=dm}
\end{lemma}
\begin{proof}[Proof of Lemma~\ref{lemma:sm=dm}]
Although there are other related proofs for this result~\citep{fisher,score}, here we provide a simpler reasoning based on the Fokker-Planck equation for stochastic differential equations (SDE). 
We consider the first-order Langevin-type SDE:
\begin{align}
d\theta_t = \frac{1}{2} \nabla_\theta \log p(\theta_t) dt + dW_t
\label{eq:sde_p}
\end{align}
It is known that $p(\theta)$ is the stationary distribution of the dynamics (\ref{eq:sde_p}) due to the Fokker-Planck equation:
\begin{align}
0 = \frac{\partial p(\theta)}{\partial t} = -\textrm{div} \bigg\{ p(\theta) \cdot \frac{1}{2} \nabla_\theta \log p(\theta) \bigg\} + \frac{1}{2} \Delta p(\theta) 
= -\frac{1}{2} \textrm{div} \nabla_\theta p(\theta) + \frac{1}{2} \Delta p(\theta) = 0
\end{align}
where $\textrm{div}(\cdot)$ is the divergence operator and $\Delta(\cdot)$ is the Laplace operator. 
Now we do the similar reasoning with $q$ in place of $p$, yielding that the SDE 
\begin{align}
d\theta_t = \frac{1}{2} \nabla_\theta \log q(\theta_t) dt + dW_t
\label{eq:sde_q}
\end{align}
admits $q(\theta)$ as its stationary distribution. Since (\ref{eq:sde_p}) and (\ref{eq:sde_q}) are identical, we have $p=q$. 
\end{proof}

\section{Additional Experiments: PosteriorDB}\label{appsec:more_expmt}



We test competing methods on the PosteriorDB benchmark~\citep{posteriordb}, a database containing a wide range of realistic target densities in Bayesian posterior inference, which provides the true scores $\nabla_\theta \log \pi(\theta)$ for various models and datasets. The database provides the Stan code~\citep{stan} for accessing the true scores, and similarly as~\citep{gsmvi}, we use the BridgeStan~\citep{bridgestan} Python interface to access the scores in our codebase. 

For this study we consider four datasets -- two Gaussian targets: \texttt{diamonds} (generalized linear models) and \texttt{arK} (time-series); two non-Gaussian targets: \texttt{gp-pois-regr} (Gaussian processes) and \texttt{low-dim-gauss-mix} (Gaussian mixtures). And we focus on two experimental settings: i) We first compare the convergence of the proposed variational inference algorithm with competing methods; 
ii) We consider noisy unbiased score experiments by injecting noise to the score. 
For the latter, the noisy score is  specifically defined as:
\begin{align}
\hat{s}(\theta) := s(\theta) + \beta \cdot ||s(\theta)||_\infty \cdot \mathcal{N}(0,I)
\label{eq:noisy_score}
\end{align}
where $\beta$ is the noise level we vary -- the larger/smaller $\beta$ implies more/less noise injected to the score (e.g., $\beta=0$ corresponds to noise-free exact score). Clearly the noise injection process (\ref{eq:noisy_score}) makes $\hat{s}(\theta)$ unbiased (i.e., $\mathbb{E}[\hat{s}(\theta)] = s(\theta) = \nabla_\theta \log \pi(\theta)$). We will control $\beta$ and see how the competing VI methods would perform for different noise levels.  

For all four datasets we adopt the Gaussian variational density family since we observe that enriching the family to the mixture of Gaussians even in the non-Gaussian datasets did not improve the approximation quality further. 

\textbf{Convergence with noise-free scores.}
In Fig.~\ref{fig:posteriordb_convergence} we show the evolution of the negative ELBO metric as the number of score calls increases (averaged over 5 runs for each method). For all datasets we see that all methods converge to the same optimal negative ELBO values. The iterative closed-form method GSM tends to converge the fastest on \texttt{gp-pois-regr} and \texttt{low-dim-gauss-mix}, but its convergence becomes slower than our approach on the other datasets. Another score-based method BaM converges as fast as GSM on \texttt{gp-pois-regr} and \texttt{arK-arK}, but it exhibits less stability in optimization. Consistently for all datasets, our method converges substantially faster than ADVI, which can be largely attributed to our objective based on score matching. 

\begin{figure}[t!]
\begin{center}
\centering
\includegraphics[trim = 5mm 5mm 6mm 4mm, clip, scale=0.205 
]{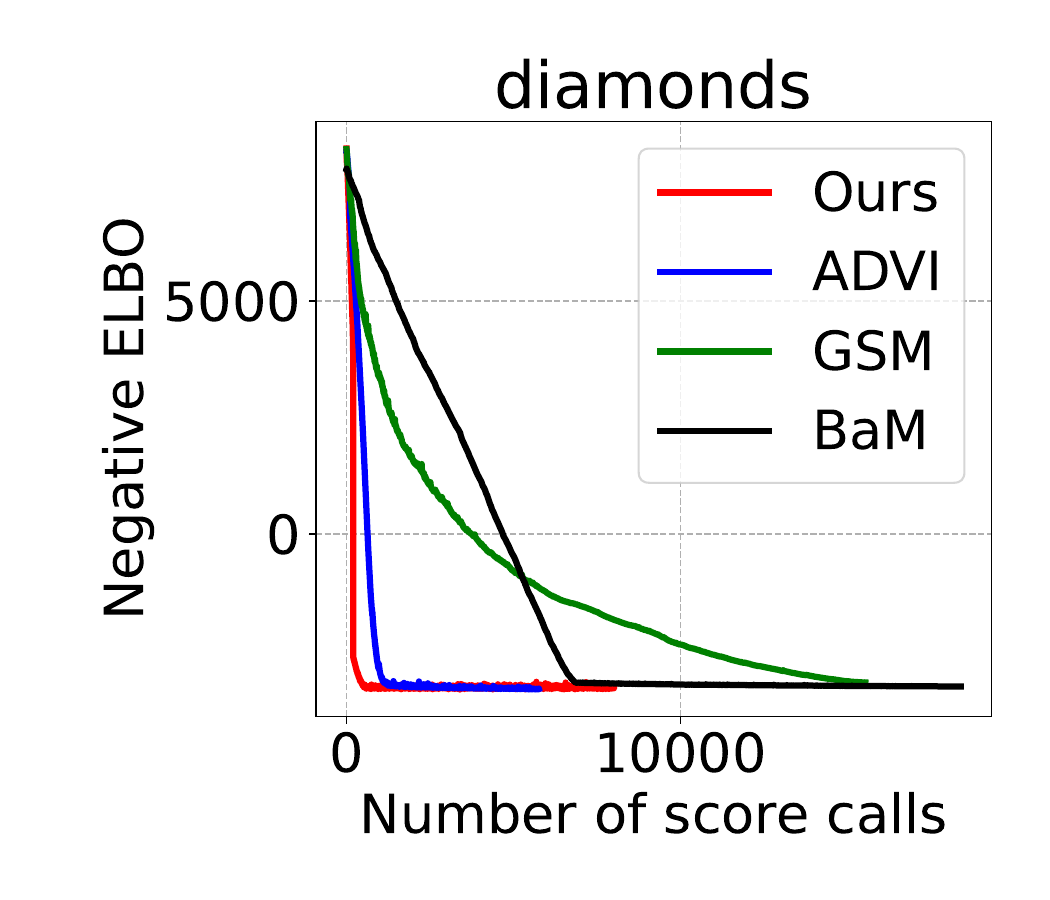}
%
\includegraphics[trim = 5mm 5mm 6mm 4mm, clip, scale=0.205 
]{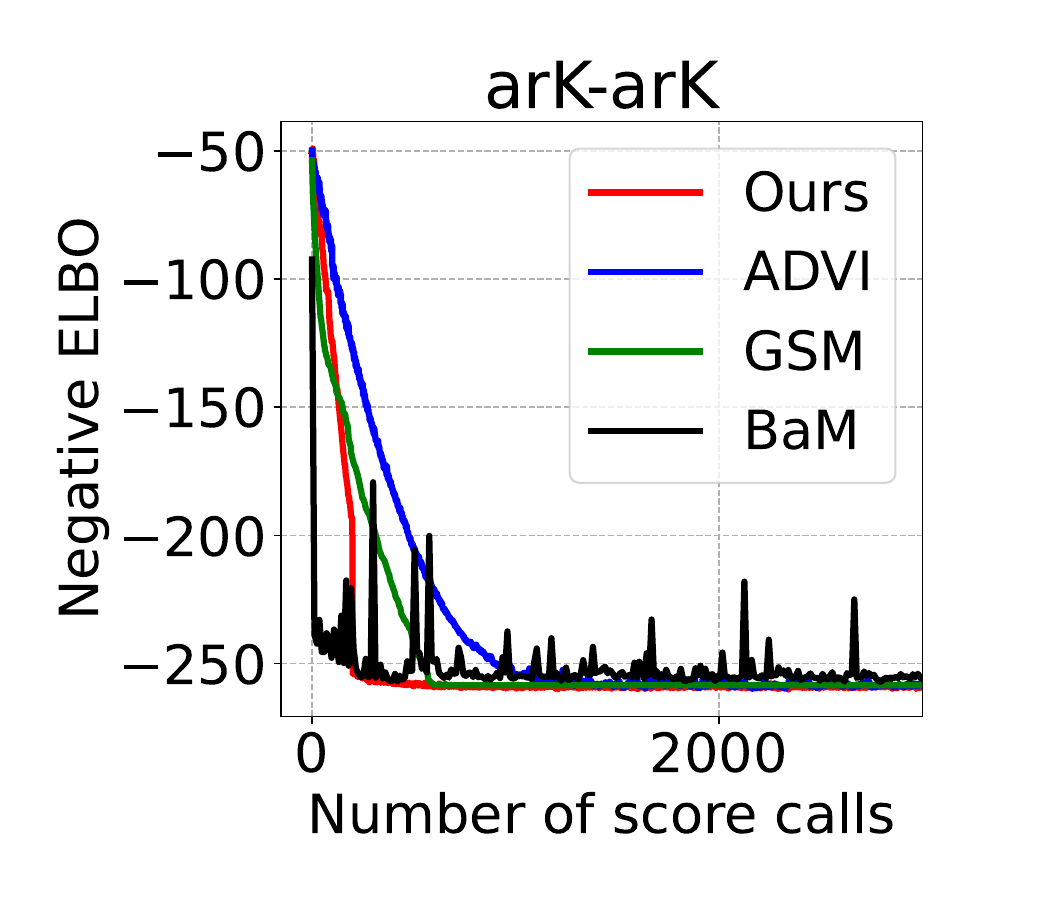}
%
\includegraphics[trim = 5mm 5mm 6mm 4mm, clip, scale=0.205 
]{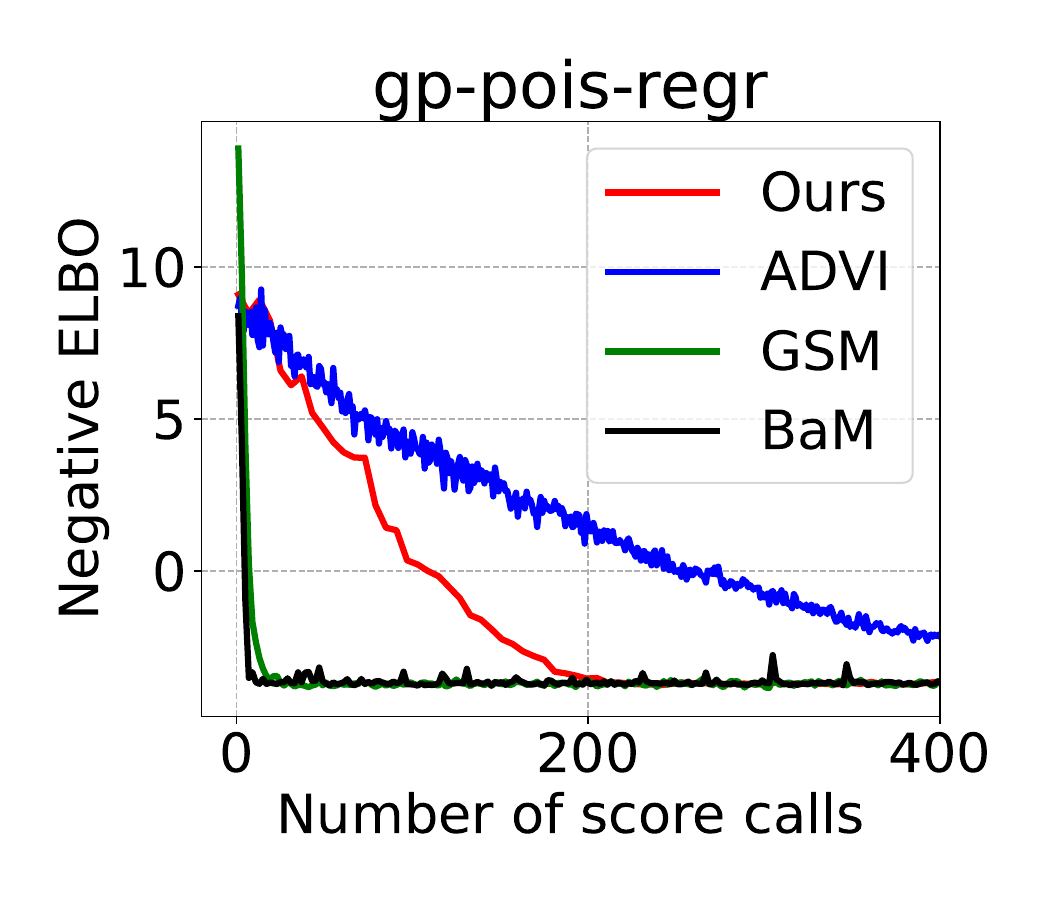}
%
\includegraphics[trim = 5mm 5mm 6mm 4mm, clip, scale=0.205 
]{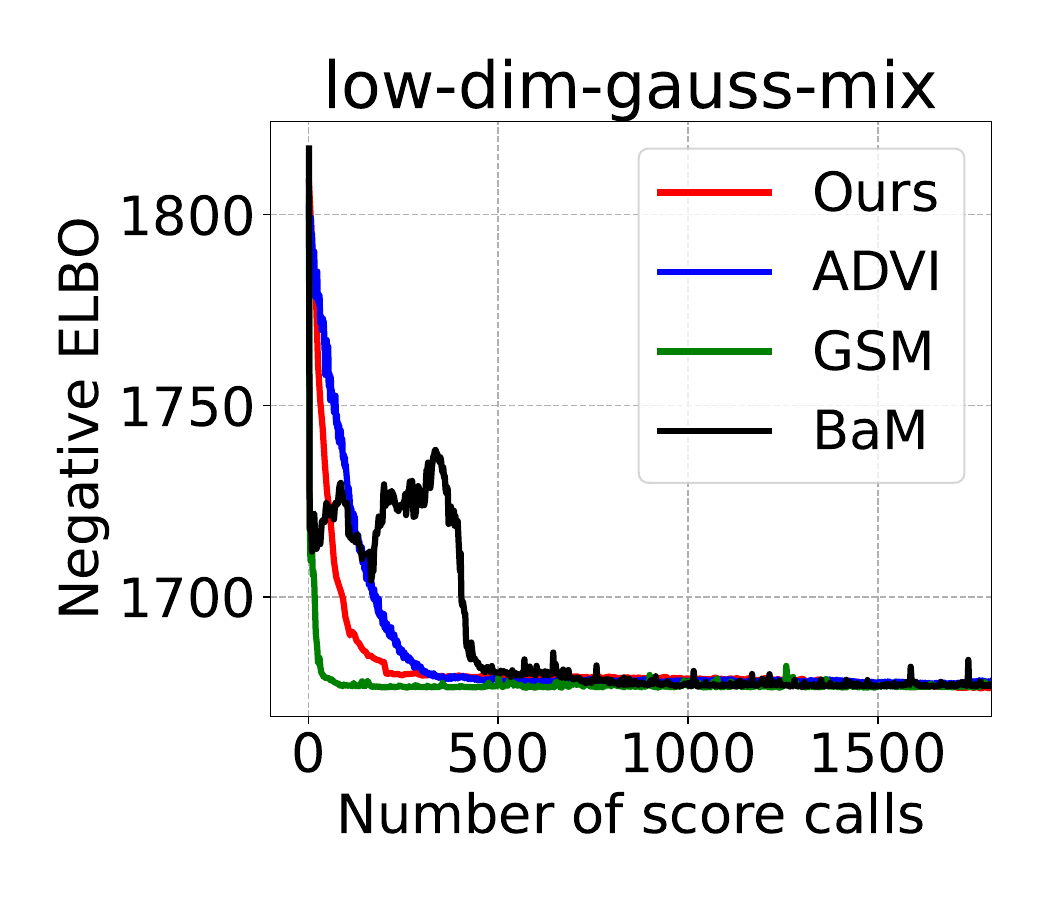}
\end{center}
\vspace{-1.5em}
\caption{
(PosteriorDB) Convergence with noise-free scores.
}
\label{fig:posteriordb_convergence}
\end{figure}

\begin{figure}
\begin{center}
\centering
\includegraphics[trim = 5mm 5mm 6mm 4mm, clip, scale=0.205 
]{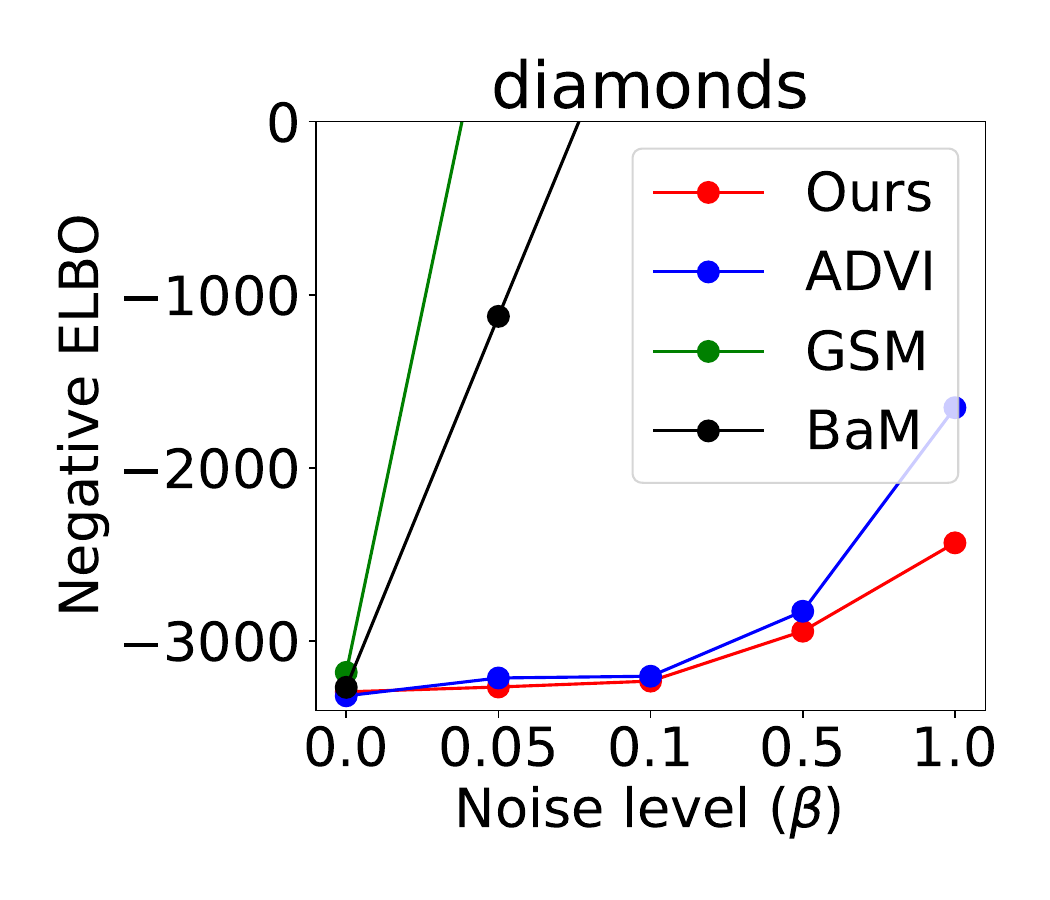}
%
\includegraphics[trim = 5mm 5mm 6mm 4mm, clip, scale=0.205 
]{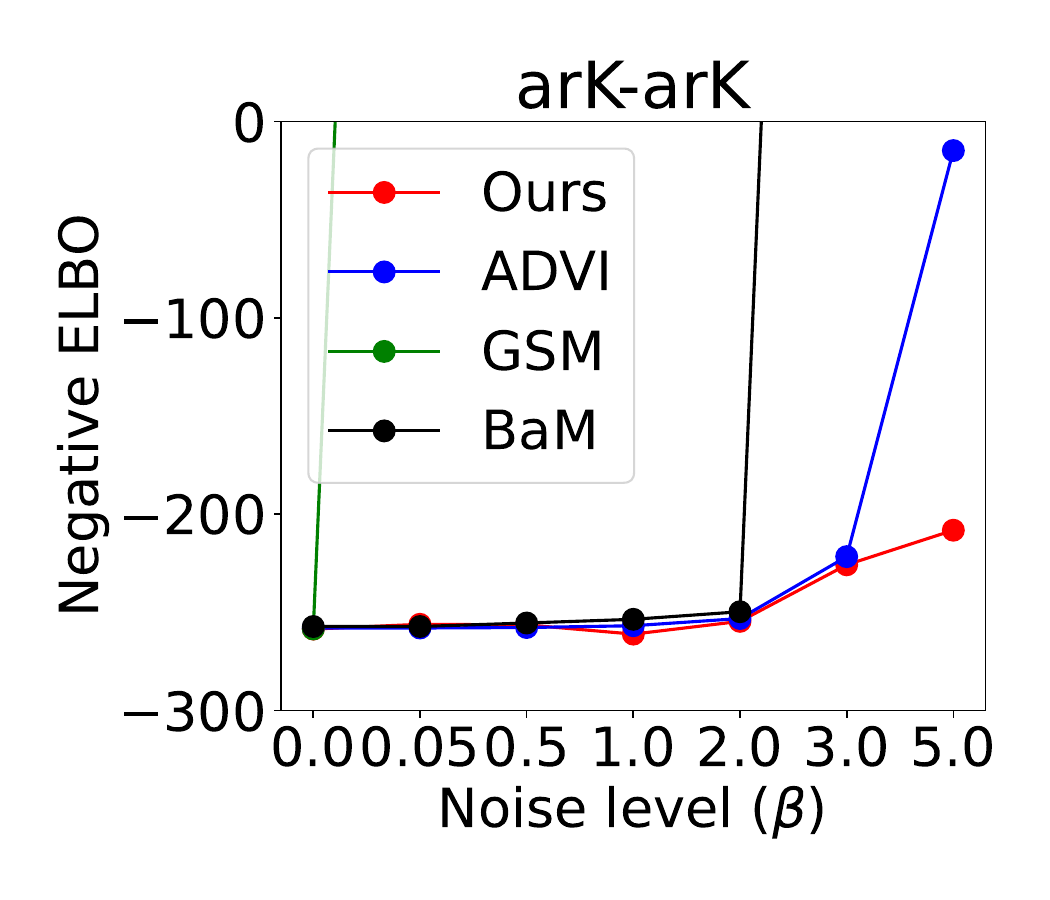}
%
\includegraphics[trim = 5mm 5mm 6mm 4mm, clip, scale=0.205 
]{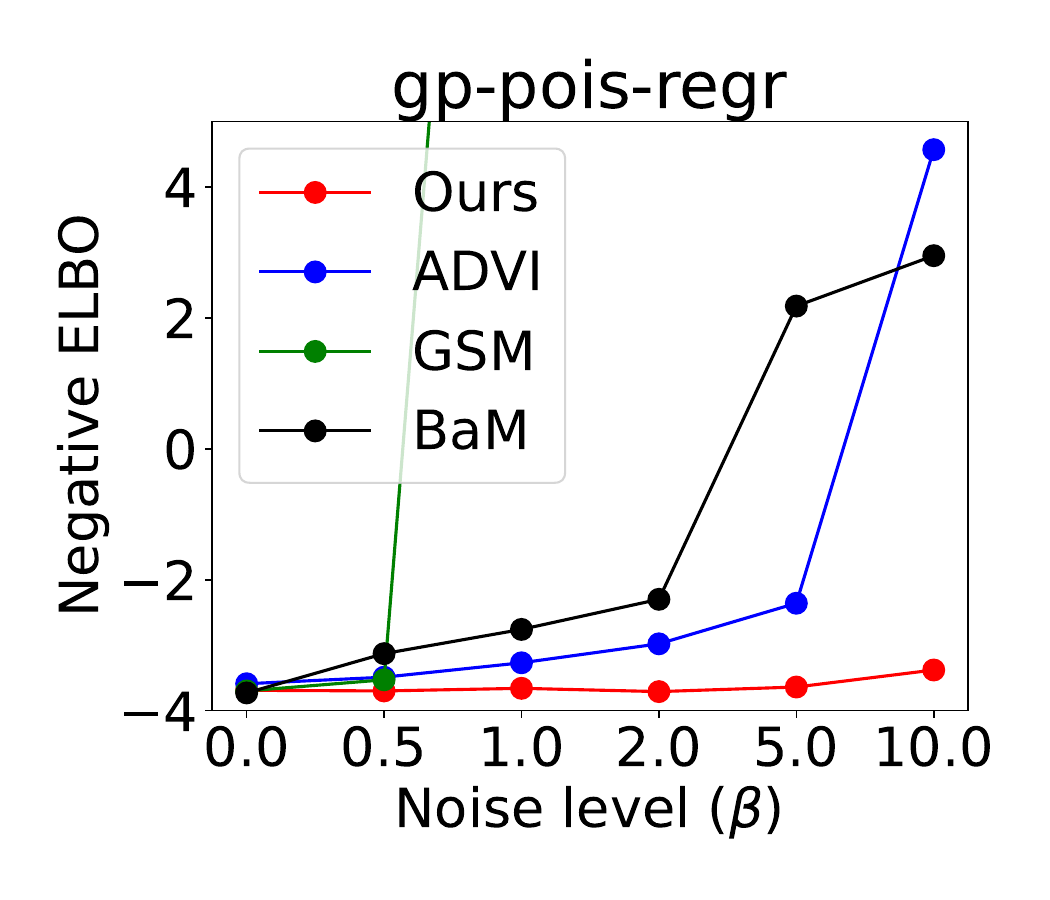}
%
\includegraphics[trim = 5mm 5mm 6mm 4mm, clip, scale=0.205 
]{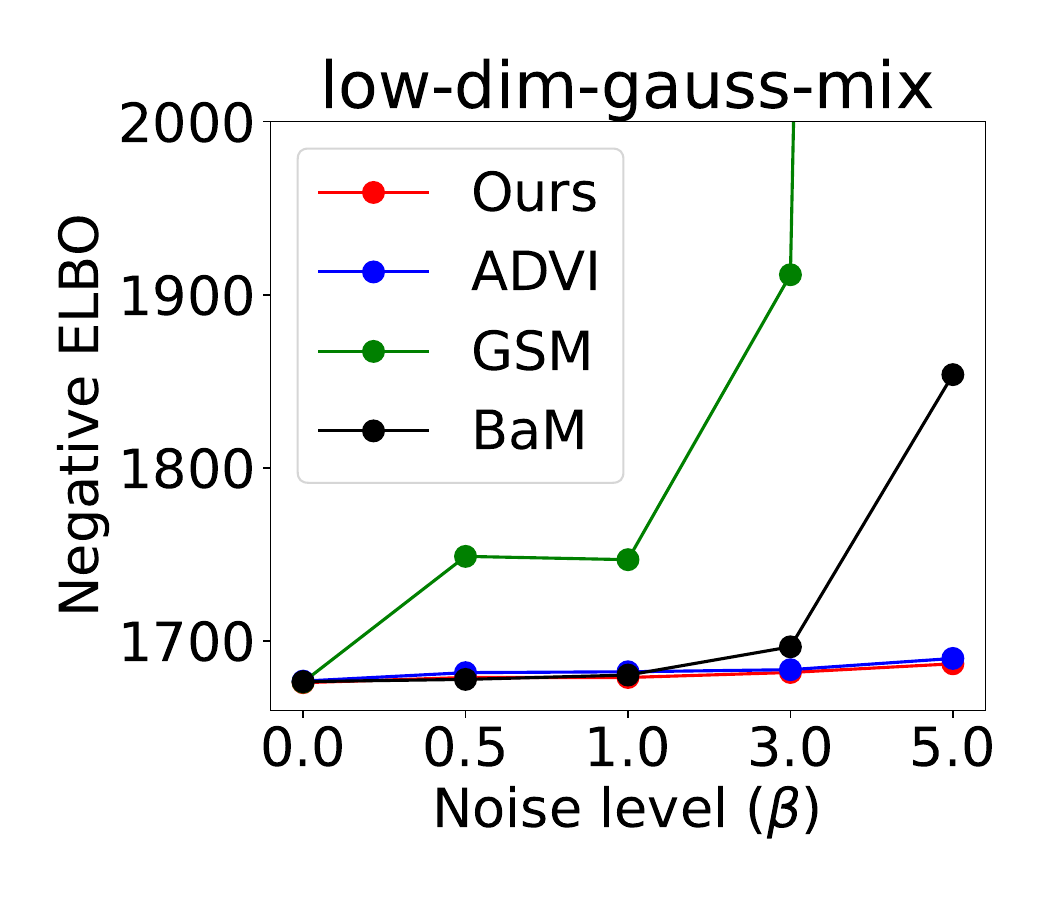}
\end{center}
\vspace{-1.5em}
\caption{
(PosteriorDB) Robustness to noise in noisy score scenarios. Each point is the negative ELBO for the final converged model. 
}
\label{fig:posteriordb_noisy}
\end{figure}

\textbf{Robustness to noise in noisy score scenarios.}
We next see how the final optimized negative ELBO values change when we inject unbiased noise to the true score as (\ref{eq:noisy_score}). By varying the noise level $\beta$ we plot the final converged values in Fig.~\ref{fig:posteriordb_noisy}. 
We see that the competing score-based methods GSM and BaM fail to converge to reasonable negative ELBO values since these methods inherently assume that the scores are noise-free. Our approach and ADVI are stochastic-gradient methods, and both of them exhibit more robust results than GSM and BaM. But it is evident that our approach is the most robust to the noise in scores even with extreme noise levels as indicated by the rightmost points in the plots. On the contrary ADVI often failed to attain near optimal values in those extreme conditions.

\section{Details of the Experiments}\label{appsec:expmt_details}

For the synthetic Gaussian and Gaussian mixture experiments, our approach takes $N=20$ (the number of steps for each iteration) and the number of samples $S=1$. ADVI uses the same $S=1$ where we choose the Gumbel-softmax temperature $0.05$ for Gaussian mixture reparametrized sampling, determined by some pilot runs. For GSM and BaM, as suggested in~\citep{gsmvi}, we use the number of (batch) samples 2. To report the forward KL metric, we use the Monte Carlo estimation with 500 samples from the target $\pi$. 
For the Gaussian mixture parameters, the mean parameters are randomly sampled from $\mathcal{N}(0,1)$, and the covariances are randomly chosen as $I + a XX^\top$ where the elements of the matrix $X$ are sampled from $\mathcal{N}(0,1)$ and $a \sim \mathcal{N}(0, 0.3)$. We repeat this procedure until we have positive definite covariances. Each element of the mixing proportions is chosen as $(1+b)$ with $b \sim \mathcal{N}(0,0.3)$ in the logit space, followed by softmax transformation.

For PosteriorDB experiments, our approach takes $N=2$ steps for each iteration and the number of samples $S=1$. ADVI uses the same $S=1$ while GSM and BaM use the number of (batch) samples 2. To compute the negative ELBO metric, we use the Monte Carlo estimation with 1000 samples. 

For the MNIST experiments, for all competing methods, we set the maximum training epochs as 100, learning rate $10^{-2}$, batch size 128, and the SGD optimizer with momentum 0.5. The number of $\theta$ samples $S=1$. The MLP network we adopt has three hidden layers with 1000 units and ReLU nonlinearity, followed by the final linear prediction head. 
For the predictive distribution $p(y^*|x^*,D) = \int p(y^*|x^*,\theta) p(\theta|D) d\theta$, the number of samples from the posterior $p(\theta|D)$ is 5.

For the large-scale vision experiments, the number of epochs is 50, batch size 16, and the number of samples $S=1$. We take $N=2$ steps. The learning rate is $10^{-4}$ for the feature extraction part of the networks and $10^{-2}$ for the classification readout head. 
Also to turn the classification cross entropy loss into likelihood, we take the widely-used transformation for the likelihood model:
$p(D|\theta) = e^{-loss(\theta;D)/\tau}$ 
where $loss(\theta;D)$ is the cross entropy loss of the underlying deep network with parameters $\theta$ on data $D=\{(x,y)\}$, and $\tau$ is the temperature hyperparameter chosen from domain knowledge. We use $\tau=10^{-6}$ for all cases/datasets. To compute the ECE and test NLL scores, we take 5 samples from the learned posteriors. Each method is run with 5 random seeds.

For the time-series forecasting experiments, we have the same setting for our approach as the vision experiments except for the learning rate $10^{-3}$, number of epochs 10, and batch size 32. The other Koopa hyperparameters follow the default values from~\citep{koopa}. 
The six benchmark datasets are for multivariate forecasting, and they are: ECL
(UCI~\citep{uci}), ETT~\citep{ts_ett}, Exchange~\citep{ts_exchange}, ILI (CDC), Traffic (PeMS), and Weather (Wetterstation). Similarly as~\citep{koopa}, we follow the
data pre-processing and split protocols used in TimesNet~\citep{timenet}. 
To compute the test NLL scores, we sample 5 samples from the learned posterior, and formed a mixture of 5 Gaussians centered at these samples with variances chosen using the validation sets. Each method is run with 10 random seeds.

\section{Additional Experiments}\label{appsec:additional_expmts}

\subsection{Convergence on Visual Recognition Tasks}

We visualize the convergence of the our method and ADVI in Fig.~\ref{app_fig:vit_convergence}. We show the test loss convergence instead of the negative ELBO on training data since the ELBO is merely a bound on the KL divergence and as the dimensionality grows at this model scale, it may not be a good metric to capture distribution divergence effectively. We see that our algorithm consistently converges faster to lower test loss values than ADVI.

\begin{figure}[t!]
\begin{center}
\centering
\includegraphics[trim = 4mm 5mm 6mm 4mm, clip, scale=0.195 
]{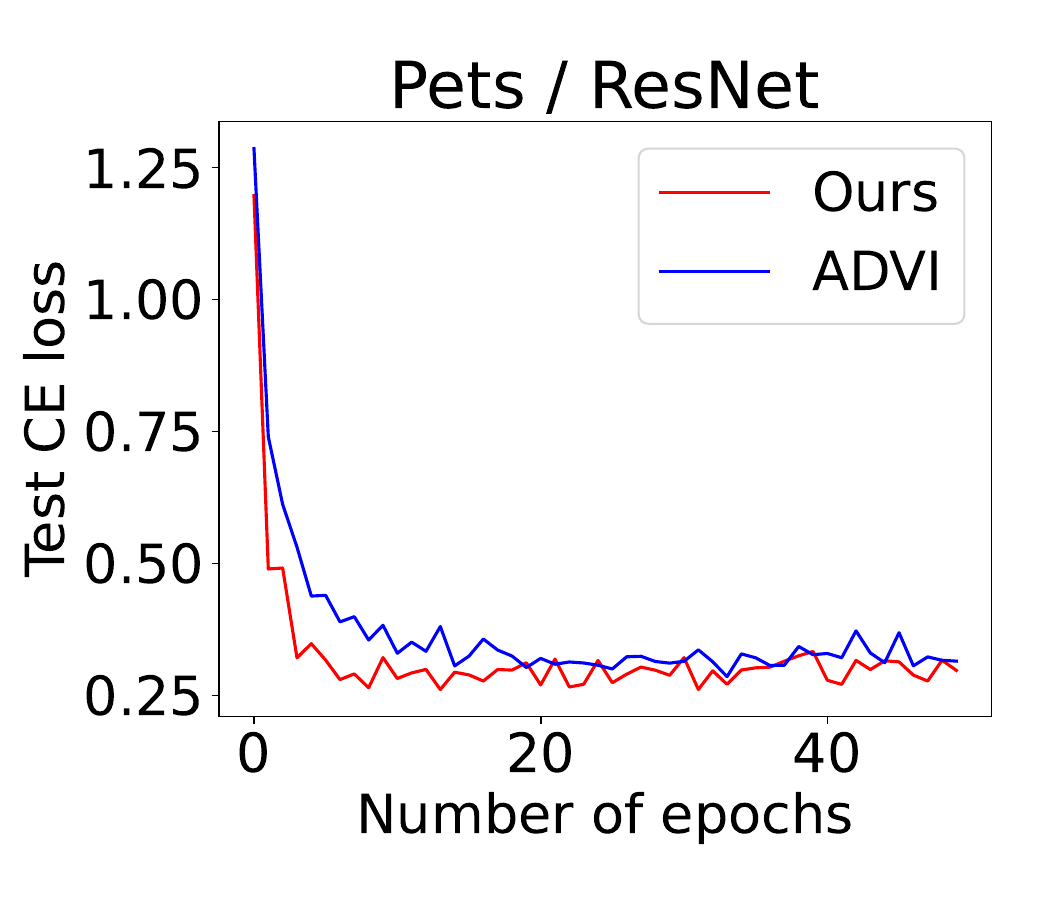}
\
\includegraphics[trim = 4mm 5mm 6mm 4mm, clip, scale=0.195 
]{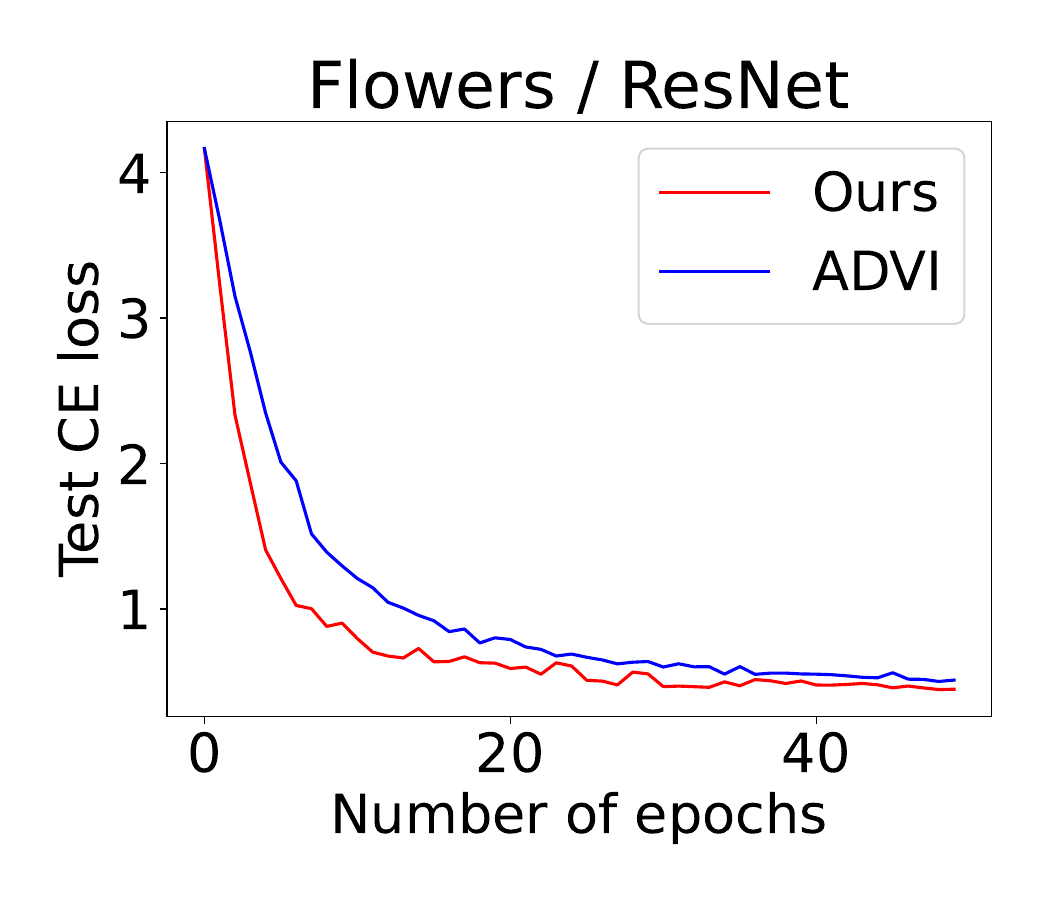}
\
\includegraphics[trim = 4mm 5mm 6mm 4mm, clip, scale=0.195 
]{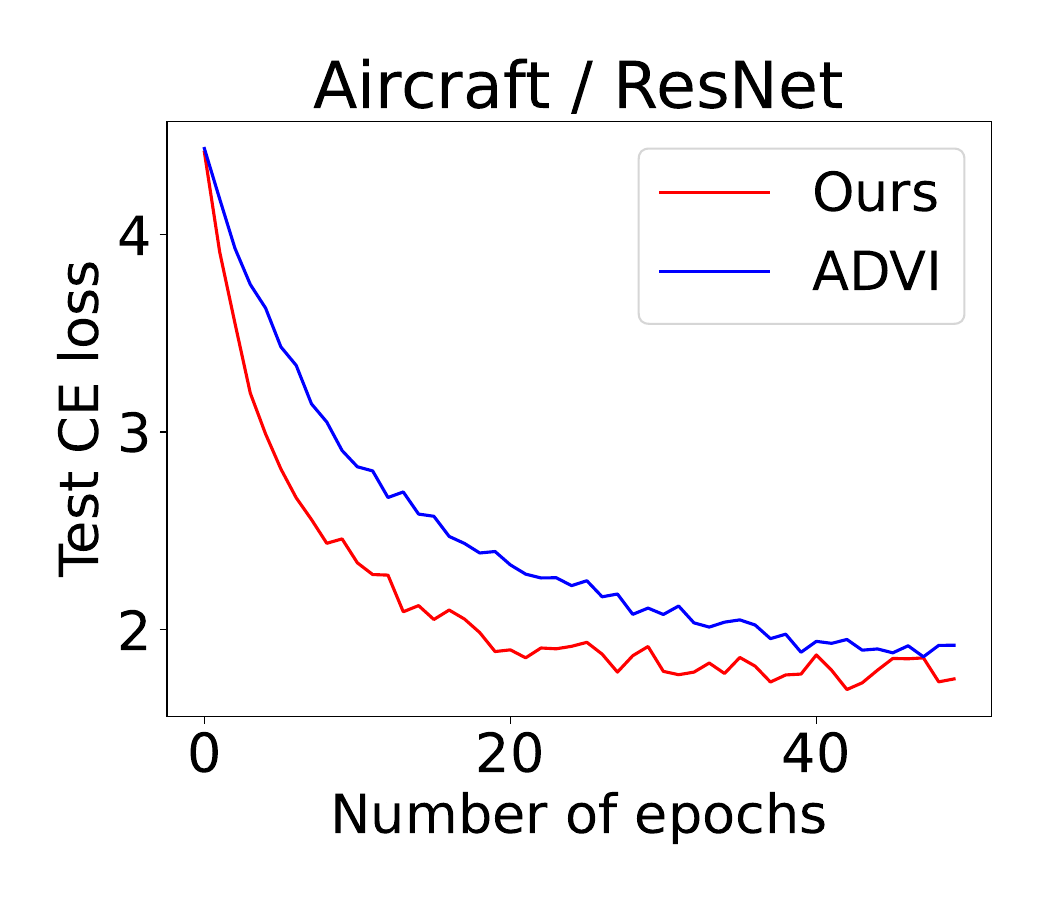}
\
\includegraphics[trim = 4mm 5mm 6mm 4mm, clip, scale=0.195 
]{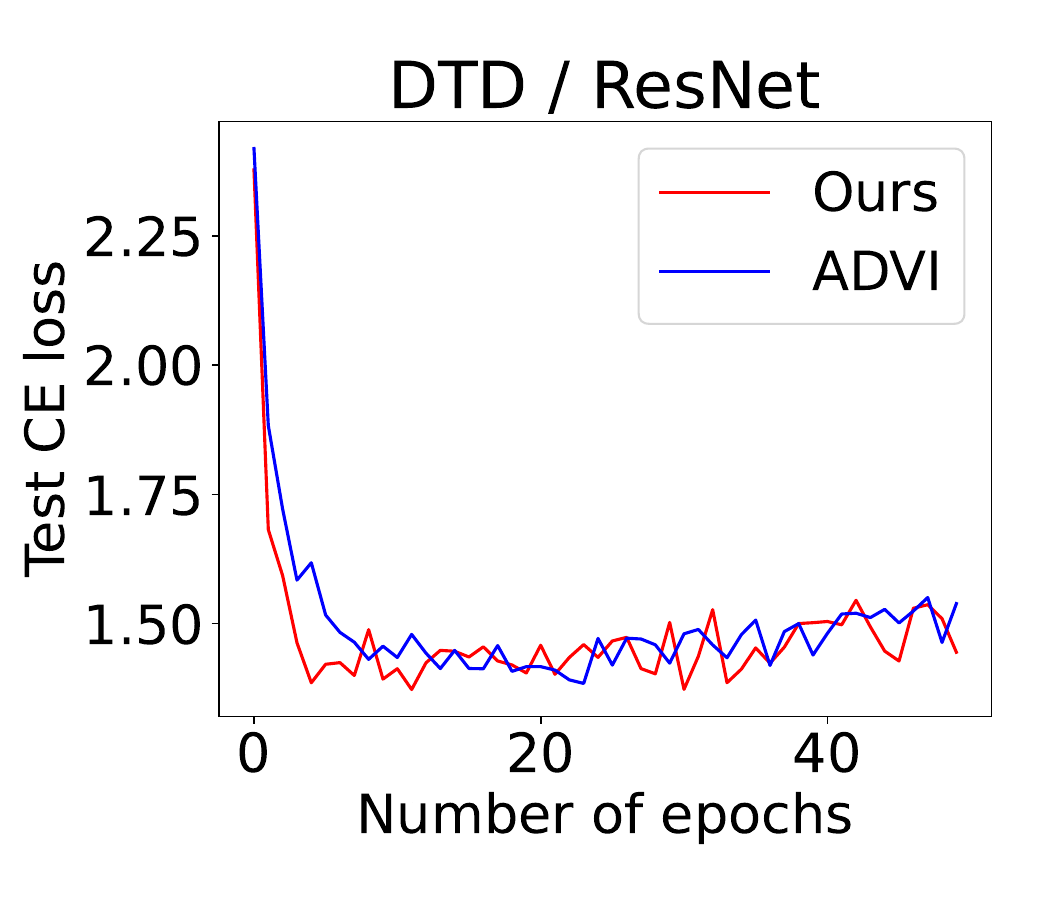}
\\
\includegraphics[trim = 4mm 5mm 6mm 4mm, clip, scale=0.195 
]{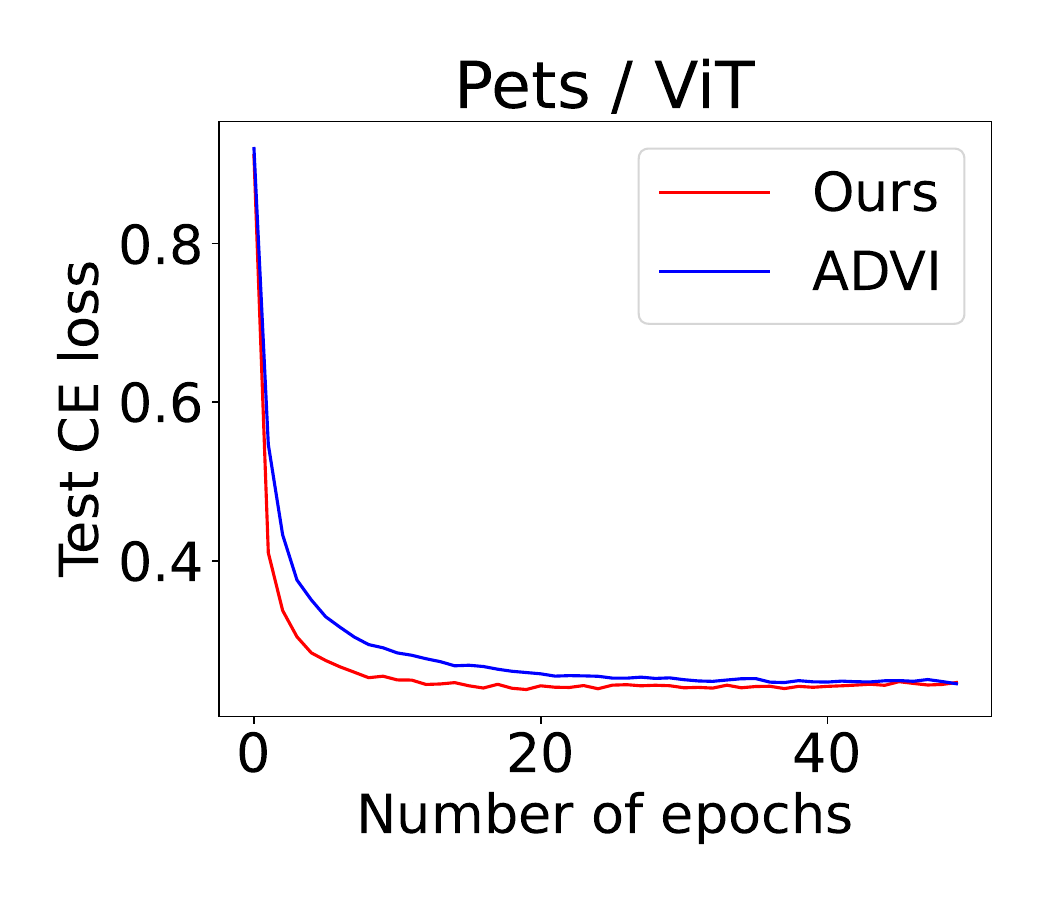}
\
\includegraphics[trim = 4mm 5mm 6mm 4mm, clip, scale=0.195 
]{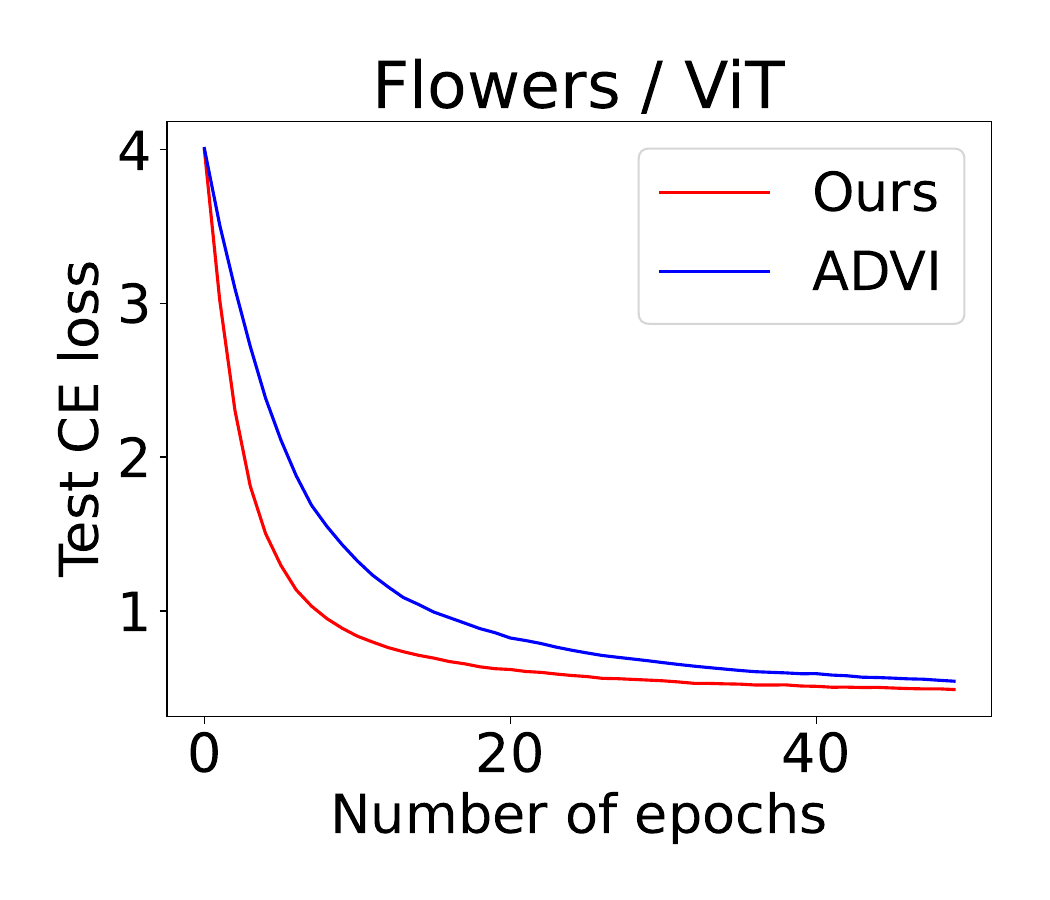}
\
\includegraphics[trim = 4mm 5mm 6mm 4mm, clip, scale=0.195 
]{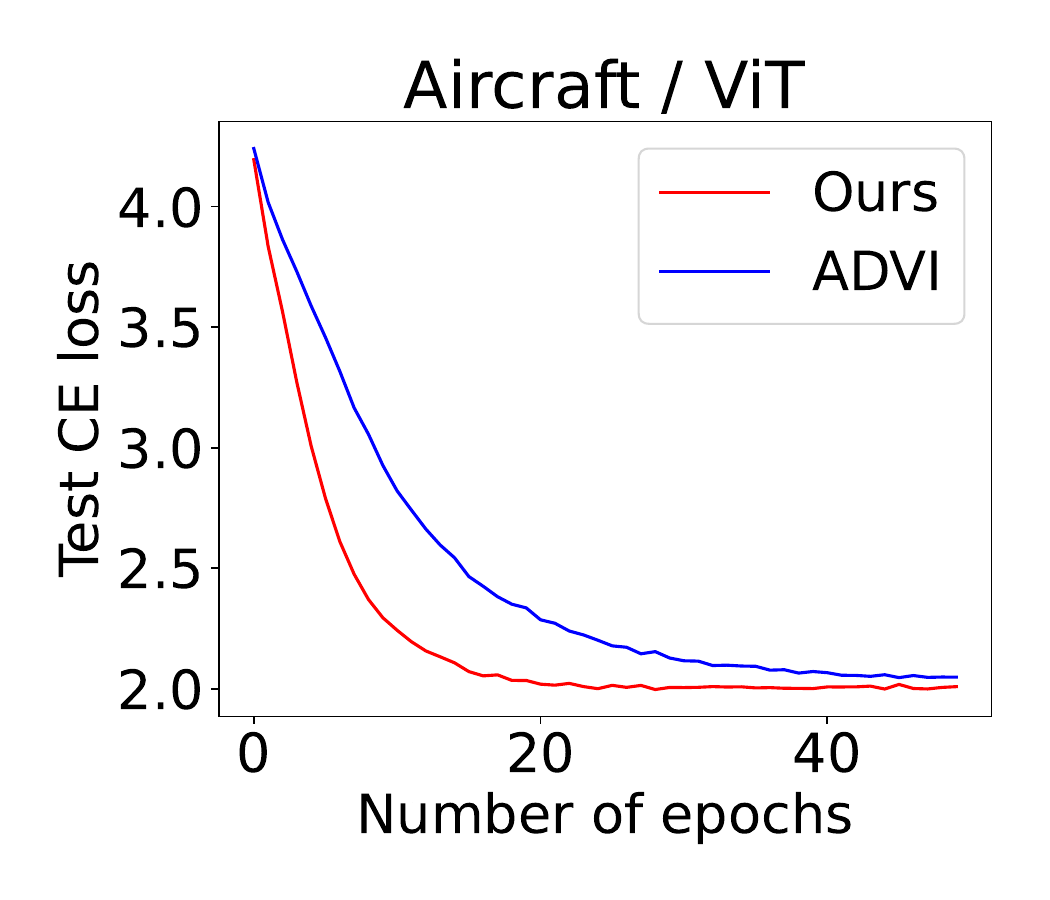}
\
\includegraphics[trim = 4mm 5mm 6mm 4mm, clip, scale=0.195 
]{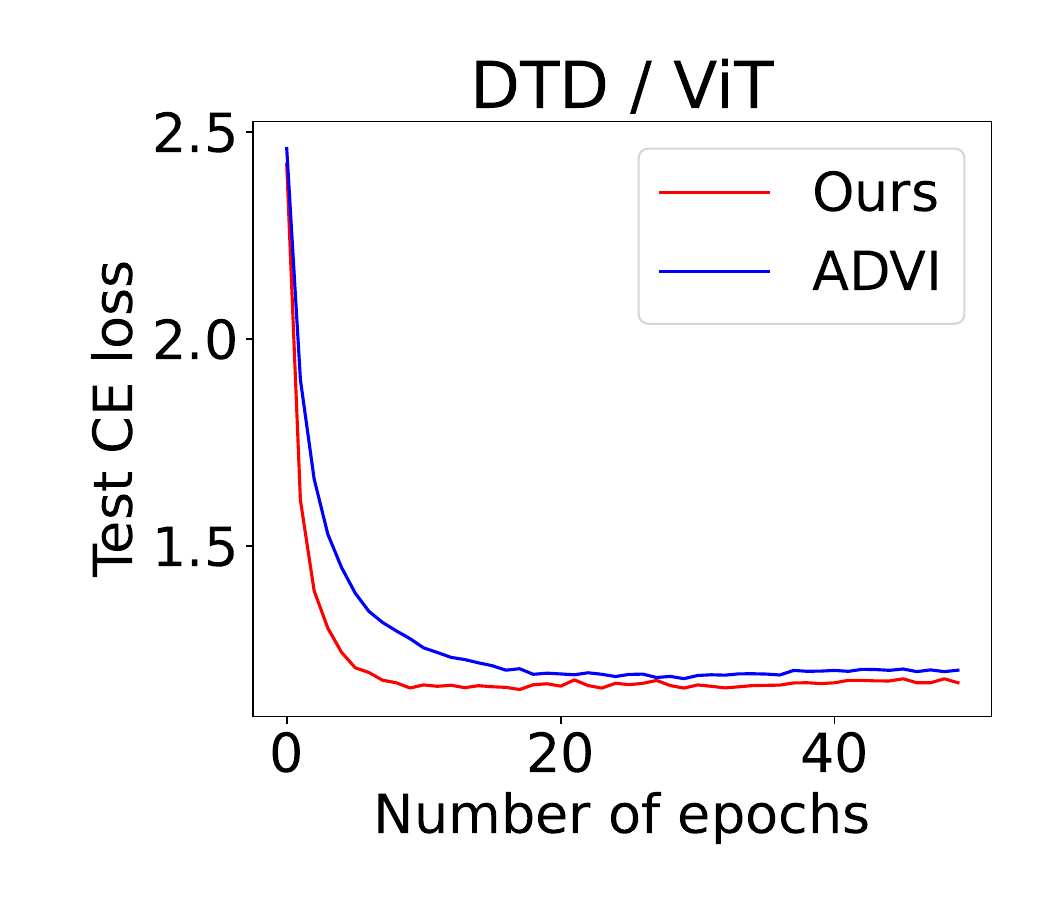}
\end{center}
\vspace{-1.3em}
\caption{
(Large-scale Bayesian deep learning) Convergence in test cross entropy losses.
}
\label{app_fig:vit_convergence}
\end{figure}


\subsection{Running Time vs.~Number of MC Samples}

In the main paper, the default number of MC samples is 5. We vary the number of MC samples, and measure the per-batch running time for the Pets/ViT-L-32 experiments. See 
Table~\ref{apptab:running_time_mc_samples} for the results. 

\begin{table*}[t!]
\setlength{\tabcolsep}{12.35pt}
\caption{
Running time vs.~the number of MC samples on Pets/ViT-L-32. 
}
\centering
\begin{scriptsize}
\centering
\begin{tabular}{ccccccc}
\toprule
Number of MC samples &
1 & 2 & 5 & 10 & 20 & 50 \\ 
\midrule
Per-batch time (seconds) & 0.1013 & 0.1804 & 0.4323 & 0.8119 & 1.5825 & 3.8955 \\
\bottomrule
\vspace{-0.5em}
\end{tabular}
\end{scriptsize}
\label{apptab:running_time_mc_samples}
\end{table*}

\subsection{Statistical Significance}

To verify statistical significance of improvement in our real data experiments, we have conducted the paired $t$ test. Table~\ref{apptab:paired_t_test} shows the $p$-values between our method and the ADVI method on the NLL scores (the lower the better) for 14 sets of real data experiments.
We see that among 14 comparisons, we have 9 cases ($\sim 64\%$) where the $p$-value is less than 0.05, implying statistical significance of the improvement made by our approach against the ELBO-based ADVI.

\begin{table*}[t!]
\setlength{\tabcolsep}{12.35pt}
\caption{
Paired $t$-test between our approach and ADVI on real data experiments.
}
\centering
\begin{scriptsize}
\centering
\begin{tabular}{ccccc}
\toprule
(a) Vision/ResNet-101 & Pets & Flowers & Aircraft & DTD \\ 
\midrule
$p$-value & 0.0019 & 0.2420 & $1.23 \times 10^{-5}$ & 0.0005 \\
\bottomrule
\vspace{-0.5em}
\end{tabular}
\begin{tabular}{ccccc}
\toprule
(b) Vision/ViT-L-32 &Pets & Flowers & Aircraft & DTD \\ 
\midrule
$p$-value & 0.0004 & 0.1100 & 0.4229 & $2.31 \times 10^{-5}$ \\
\bottomrule
\vspace{-0.5em}
\end{tabular}
\begin{tabular}{ccccccc}
\toprule
Bayesian Koopa & ECL & ETTh2 & Ex.~Rate & ILI & Traffic & 
Weather \\ 
\midrule
$p$-value & 0.6524 & 0.2237 & 0.0233 & $6.83 \times 10^{-7}$ & $1.24 \times 10^{-7}$ & 0.0007 \\
\bottomrule
\vspace{-0.5em}
\end{tabular}
%
\end{scriptsize}
\label{apptab:paired_t_test}
\end{table*}

\subsection{Synthetic Model Mismatch Experiments: $\pi(\theta)=$ 
Gaussian Mixture, $q(\theta)=$ Gaussian}\label{app_sec:model_mismatch2}

We have conducted additional experiment of the model mismatch where the variational density family doesn't include the true posterior. In the synthetic setting, 
we set $\dim(\theta)=3$, $\pi(\theta)=$ order-2 Gaussian mixture, and $q(\theta)=$ Gaussian family. The results in Table~\ref{apptab:toy_mismatch} show that all competing methods yield relatively higher forward KL scores (suboptimal solutions) due to the inferior model class, but our approach has achieved the best score among others.

\begin{table*}[t!]
\setlength{\tabcolsep}{12.35pt}
\caption{
Synthetic model mismatch 
with $\pi(\theta)=$ order-2 Gaussian mixture, $q(\theta)=$ Gaussian.
}
\centering
\begin{scriptsize}
\centering
\begin{tabular}{ccccc}
\toprule
& Ours & ADVI & GSM & BaM \\ 
\midrule
forward KL & 0.8339 & 0.9461 & 0.9965 & 1.1914 \\
\bottomrule
\vspace{-0.5em}
\end{tabular}
\end{scriptsize}
\label{apptab:toy_mismatch}
\end{table*}

\begin{table*}[t!]
\setlength{\tabcolsep}{12.35pt}
\caption{(Gaussian target experiment) 
Converged FKL scores and required numbers of iterations. 
}
\centering
\begin{scriptsize}
\centering
\begin{tabular}{cccc}
\toprule
& Ours & Fisher~\citep{fisher} & Score~\citep{score} \\ 
\midrule
Final FKL score & $< 10^{-5}$ & $5.6 \times 10^{-4}$ & $< 10^{-5}$ \\
First iteration$\#$ when FKL $< 10^{-2}$ & 260 & 273 & 284 \\
First iteration$\#$ when FKL $< 10^{-3}$ & 320 & 358 & 362 \\
\bottomrule
\vspace{-0.5em}
\end{tabular}
\end{scriptsize}
\label{apptab:fisher_score}
\end{table*}

\subsection{Comparison with Fisher and Score Methods on Synthetic Data Experiments}

In Sec.~\ref{sec:expmt_gaussian} our approach was mainly compared with ADVI, GSM, and BaM. We additionally conducted experiments with Fisher and Score methods. In this small-scale synthetic data, both methods work equally well, and the final converged values are comparable to our approach. In Table~\ref{apptab:fisher_score}, we show the convergence results by reporting the numbers of iterations required to hit the target forward KL scores. 
Note however that these methods work well on small-scale data, but do not scale to large-scale data and large-scale Bayesian deep networks.



\end{document}